\documentclass{article}

\usepackage{microtype}
\usepackage{graphicx}
\usepackage{subfigure}
\usepackage{booktabs} 
\usepackage{multirow}

\usepackage{comment}
\usepackage{bm}
\usepackage{url}            
\usepackage{booktabs}       
\usepackage{amsfonts}       
\usepackage{amsmath,amssymb}
\usepackage{empheq}
\usepackage{amsthm}
\usepackage{nicefrac}       
\usepackage{microtype}      
\usepackage{enumerate,natbib,mathrsfs}
\usepackage{tablefootnote}
\usepackage{wrapfig}
\usepackage{graphicx}
\usepackage{subfigure}
\usepackage{enumitem}
\usepackage{multirow}
\usepackage{extarrows}
\usepackage[OT1]{fontenc}
\usepackage[bookmarks=false]{hyperref}

\usepackage{amsmath}

\DeclareMathOperator*{\argmin}{arg\,min}

\usepackage{hyperref}

\hypersetup{
  pdffitwindow=true,
  pdfstartview={FitH},
  pdfnewwindow=true,
  colorlinks,
  linktocpage=true,
  linkcolor=red,
  urlcolor=red,
  citecolor=darkblue
}

\usepackage{hyperref}
\hypersetup{colorlinks=true,citecolor=darkblue,linkcolor=red}
\usepackage{lipsum}

\newcommand{\bxi}{{\boldsymbol \xi}}

\newcommand{\bu}{{\boldsymbol u}}
\newcommand{\bx}{{\boldsymbol x}}

\newcommand{\by}{{\boldsymbol y}}

\newcommand{\bw}{{\mathbf{w}}} 
\newcommand{\bbf}{{\boldsymbol f}}

\newcommand{\bLambda}{{\boldsymbol \Lambda}}
\newcommand{\bA}{{\mathrm{\mathbf{A}}}}
\newcommand{\bTheta}{{\mathrm{\mathbf{\Theta}}}}
\newcommand{\bB}{{\mathrm{\mathbf{B}}}}
\newcommand{\bC}{{\mathrm{\mathbf{C}}}}
\newcommand{\bD}{{\mathrm{\mathbf{D}}}}
\newcommand{\bM}{{\mathrm{\mathbf{M}}}}
\newcommand{\bH}{{\mathrm{\mathbf{H}}}}
\newcommand{\bL}{{\mathrm{\mathbf{L}}}}

\newcommand{\bI}{{\mathrm{\mathbf{I}}}}

\newcommand{\MX}{{\mathcal{X}}}

\newcommand{\bepsilon}{{\boldsymbol \epsilon}}
\newcommand{\bmu}{{\boldsymbol \mu}}

\newcommand{\bSigma}{{\boldsymbol \Sigma}}

\newcommand{\dd}{\mathcal{\dagger}}

\newcommand{\ea}{\end{array}}
\newcommand{\ee}{\end{equation}}
\newcommand{\bea}{\begin{eqnarray}}
\newcommand{\eea}{\end{eqnarray}}
\newcommand{\beaa}{\begin{eqnarray*}}
\newcommand{\eeaa}{\end{eqnarray*}}

\def\E{\mathbb{E}}

\def\bD{{\bf D}}

\def\bG{{\bf G}}
\def\bx{{\bf x}}
\def\by{{\bf y}}
\def\bz{{\bf z}}

\def\qed{ \hfill\qedsymbol}

\newcommand{\basa}{\begin{assumption}}
\newcommand{\easa}{\end{assumption}}

\newcommand{\bas}{\begin{assum}}
\newcommand{\eas}{\end{assum}}

\def\dd{\mathrm{d}}

\def\limP2{\,\mathop{\buildrel \Pi_2\over\longrightarrow\,}}

\def\1{{\bf 1}}
\def\by{{\bf y}}

\def\:{\!:\!}

\newtheorem{assump}{Assumption}

\usepackage{algorithm}
\usepackage{algorithmic}

\newtheorem{theorem}{Theorem}
\newtheorem{corollary}{Corollary}
\newtheorem{lemma}{Lemma}
\newtheorem{proposition}{Proposition}
\newtheorem{assumption}{Assumption}

\usepackage{hyperref}

\usepackage[accepted]{icml2024}

\usepackage{amsmath}
\usepackage{amssymb}
\usepackage{mathtools}
\usepackage{amsthm}

\usepackage[capitalize,noabbrev]{cleveref}

\theoremstyle{plain}
\theoremstyle{definition}

\theoremstyle{remark}

\usepackage[textsize=tiny]{todonotes}

\definecolor{darkblue}{rgb}{0.0, 0.0, 0.55}
\definecolor{dark2blue}{rgb}{0.0, 0.0, 0.4}
\definecolor{darkred}{rgb}{0.55, 0.0, 0.0}
\definecolor{darkgreen}{rgb}{0, 0.5, 0.0}
\icmltitlerunning{Variational Schr\"odinger Diffusion Models}

\begin{document}

\twocolumn[
\icmltitle{Variational Schr\"odinger Diffusion Models}



\icmlsetsymbol{equal}{*}

\begin{icmlauthorlist}
\icmlauthor{Wei Deng}{equal,ms} 
\icmlauthor{Weijian Luo}{equal,pku}
\icmlauthor{Yixin Tan}{equal,duke} \\
\icmlauthor{Marin Bilo\v s}{ms} 
\icmlauthor{Yu Chen}{ms}
\icmlauthor{Yuriy Nevmyvaka}{ms}
\icmlauthor{Ricky T. Q. Chen}{fair}
\end{icmlauthorlist}

\icmlaffiliation{ms}{Machine Learning Research, Morgan Stanley, New York}
\icmlaffiliation{pku}{Peking University}
\icmlaffiliation{duke}{Duke University}
\icmlaffiliation{fair}{Meta AI (FAIR), New York}

\icmlcorrespondingauthor{Wei Deng}{weideng056@gmail.com}

\icmlkeywords{Machine Learning, ICML}

\vskip 0.3in
]



\printAffiliationsAndNotice{\icmlEqualContribution} 

\begin{abstract}

Schr\"odinger bridge (SB) has emerged as the go-to method for optimizing transportation plans in diffusion models. However, SB requires estimating the intractable forward score functions, inevitably resulting in the \emph{costly} implicit training loss based on simulated trajectories. To improve the scalability while preserving efficient transportation plans, we leverage variational inference to linearize the forward score functions (variational scores) of SB and restore \emph{simulation-free} forward processes in training backward scores. We propose the variational Schr\"odinger diffusion model (VSDM), where the forward process is a multivariate diffusion and the variational scores are adaptively optimized for efficient transport. Theoretically, we use stochastic approximation to prove the convergence of the variational scores and show the convergence of the adaptively generated samples based on the optimal variational scores. Empirically, we test the algorithm in simulated examples and observe that VSDM is efficient in generations of \emph{anisotropic} shapes and yields \emph{straighter} sample trajectories compared to the single-variate diffusion. 
We also verify the scalability of the algorithm in real-world data and achieve competitive unconditional generation performance in CIFAR10 and conditional generation in time series modeling. Notably, VSDM no longer depends on warm-up initializations and has become tuning-friendly in training large-scale experiments.

\end{abstract}

\section{Introduction}

Diffusion models have showcased remarkable proficiency across diverse domains, spanning large-scale generations of image, video, and audio, conditional text-to-image tasks, and adversarial defenses \citep{SGMS_beat_GAN, imagen_video, DiffWave, text_2_image, zhang2024enhancing}. The key to their scalability lies in the closed-form updates of the forward process, highlighting both statistical efficiency \citep{Koehler_Heckett_Risteski} and diminished dependence on dimensionality \citep{dim_free_doucet}. Nevertheless, diffusion models lack a distinct guarantee of optimal transport (OT) properties \citep{Lavenant_Santambrogio_22} and often necessitate costly evaluations to generate higher-fidelity content \citep{DDPM, Progressive_distillation, DPMsolver, xue2024sa, luo2023comprehensive}.

Alternatively, the Schrödinger bridge (SB) problem \citep{leonard_14, Chen16, Pavon_CPAM_21, Caluya21, DSB}, initially rooted in quantum mechanics \citep{leonard_14}, proposes optimizing a stochastic control objective through the use of forward-backward stochastic differential equations (FB-SDEs) \citep{forward_backward_SDE}. The alternating solver gives rise to the iterative proportional fitting (IPF) algorithm \citep{Kullback_68, IPF_95} in dynamic optimal transport \citep{Villani03, 
Compute_OT}. Notably, the intractable forward score function plays a crucial role in providing theoretical guarantees in optimal transport \citep{provably_schrodinger_bridge, reflected_schrodinger_bridge}. However, it simultaneously sacrifices the simulation-free property and largely relies on warm-up checkpoints for conducting large-scale experiments \citep{DSB, forward_backward_SDE}. A natural follow-up question arises: 

\vspace{-0.03in}
\begin{center}
    {\it Can we train diffusion models with efficient transport?}
\end{center}
\vspace{-0.03in}

To this end, we introduce the variational Schr\"odinger diffusion model (VSDM). Employing variational inference \citep{blei_VI}, we perform a locally linear approximation of the forward score function, and denote it by the variational score. The resulting linear \emph{forward} stochastic differential equations (SDEs) naturally provide a \emph{closed-form update, significantly enhancing scalability}. Compared to the single-variate score-based generative model (SGM), VSDM is a multivariate diffusion \citep{multivariateDM}. Moreover, hyperparameters are adaptively optimized for \emph{more efficient transportation plans} within the Schr\"odinger bridge framework \citep{forward_backward_SDE}.

Theoretically, we leverage stochastic approximation \citep{RobbinsM1951} to demonstrate the convergence of the variational score to the optimal local estimators. Although the global transport optimality is compromised, the notable \textcolor{dark2blue}{\emph{simulation-free}} speed-ups in training the backward score render the algorithm particularly attractive for training various generation tasks from scratch. Additionally, the efficiency of simulation-based training for the linearized variational score significantly improves owing to computational advancements in convex optimization. We validate the strength of VSDM through simulations, achieving compelling performance on standard image generation tasks. Our contributions unfold in four key aspects: 
\vspace{-0.03in}
\begin{itemize}
    \item We introduce the variational Schrödinger diffusion model (VSDM), a multivariate diffusion with \emph{optimal variational scores} guided by optimal transport. Additionally, the training of backward scores is \emph{simulation-free} and becomes much more scalable.
    \vspace{-0.03in}
    \item We study the convergence of the variational score using stochastic approximation (SA) theory, which can be further generalized to a class of state space diffusion models for future developments.
    \vspace{-0.03in}
    \item VSDM is effective in generating data of \emph{anisotropic} shapes and motivates \emph{straighter} transportation paths via the optimized transport.
    \vspace{-0.03in}
    \item VSDM achieves competitive unconditional generation on CIFAR10 and conditional generation in time series modeling without reliance on warm-up initializations. 
\end{itemize}


\section{Related Works}
\label{related_works_main}

\paragraph{Flow Matching and Beyond} \citet{flow_matching} utilized the McCann displacement interpolation \citep{McCann_97} to train simulation-free CNFs to encourage straight trajectories. Consequently, \citet{Multisample_flow_matching,CFM_Tong} proposed straightening by using minibatch optimal transport solutions. Similar ideas were achieved by \citet{Rectified, Rectified_group} to iteratively rectify the interpolation path. \citet{Albergo_stochastic_interpolants, Albergo_unified_framework} developed the stochastic interpolant approach to unify both flow and diffusion models. However, \emph{``straighter'' transport} maps may not imply \emph{optimal} transportation plans in general and the couplings are still not effectively optimized. 

\paragraph{Dynamic Optimal Transport} \citet{neural_ode_kinetic, ot_flow} introduced additional regularization through optimal transport to enforce straighter trajectories in CNFs and reduce the computational cost. \citet{DSB, forward_backward_SDE, SBP_max_llk} studied the dynamic Schr\"odinger bridge with guarantees in entropic optimal transport (EOT) \citep{provably_schrodinger_bridge}; \citet{SB_matching, Peluchetti23, SB_matching_momentum} generalized bridge matching and flow matching based EOT and obtained smoother trajectories, however, \emph{scalability} remains a significant concern for Schrödinger-based diffusions.

\vspace{-0.05in}
\section{Preliminaries}

\subsection{Diffusion Models}
The score-based generative models (SGMs) \citep{DDPM, score_sde} first employ a forward process \eqref{SGM-SDE-f} to map data to an approximate Gaussian and subsequently reverse the process in Eq.\eqref{SGM-SDE-b} 
to recover the data distribution. 
\begin{subequations}\label{vanilla_diffusion_model}
\begin{align}
\dd \overrightarrow\bx_t&={{\bbf_t(\overrightarrow\bx_t) \dd t}+\sqrt{\beta_t} \dd \overrightarrow \bw_t}  \label{SGM-SDE-f}\\
\dd \overleftarrow\bx_t&=\footnotesize{\left[\bbf_t(\overleftarrow\bx_t)-\beta_t \nabla \log \rho_t\left(\overleftarrow\bx_t\right)\right] \dd t+\sqrt{\beta_t} \dd \overleftarrow{\bw}_t}, \label{SGM-SDE-b}
\end{align}
\end{subequations}
where $\overleftarrow\bx_t, \overrightarrow\bx_t\in \mathbb{R}^d$; $\overrightarrow\bx_0\sim \rho_{\text {data}}$ and $\overleftarrow\bx_T\sim \rho_{\text {prior}}$; $\bbf_t$ denotes the vector field and is often set to $\bm{0}$ (a.k.a. VE-SDE) or linear in $\bx$ (a.k.a. VP-SDE); $\beta_t>0$ is the time-varying scalar; $\overrightarrow\bw_t$ is a forward Brownian motion from $t\in[0, T]$ with $\rho_T\approx \rho_{\text{prior}}$; $\overleftarrow{\bw}_t$ is a backward Brownian motion from time $T$ to $0$. The marginal density $\rho_t$ of the forward process \eqref{SGM-SDE-f} is essential for generating the data but remains inaccessible in practice due to intractable normalizing constants. 

\paragraph{Explicit Score Matching (ESM)} Instead, the conditional score function $\nabla \log \rho_{t|0}\left(\cdot\right)\equiv \nabla \log \rho_{t}\left(\cdot|\overrightarrow\bx_0\right)$ is estimated by minimizing a user-friendly ESM loss (weighted by $\lambda$) between the score estimator ${s}_t \equiv s_{\theta}(\cdot, t)$ and exact score \citep{score_sde} such that 
\begin{equation}\label{esm_loss}
    \mathbb{E}_t \big[\lambda_t\mathbb{E}_{\overrightarrow\bx_0} \mathbb{E}_{\overrightarrow\bx_t|\overrightarrow\bx_0}[\|s_{t}(\overrightarrow\bx_t) - \nabla \log \rho_{t|0}\left(\overrightarrow\bx_t\right)\|_2^2]\big].
\end{equation}
Notably, both VP- and VE-SDEs yield closed-form expressions for any $\overrightarrow\bx_t$ given $\overrightarrow\bx_0$ in the forward process \citep{score_sde}, which is instrumental for the scalability of diffusion models in real-world large-scale generation tasks.

\paragraph{Implicit Score Matching (ISM)} By integration by parts, ESM is equivalent to the ISM loss \citep{score_matching, Variational_score_matching, luo2024entropy} and the evidence lower bound (ELBO) follows
\begin{align*}
    &\log \rho_0\left(\bx_0\right)  \geq  
    \mathbb{E}_{\rho_{T|0}(\cdot)}\left[\log \rho_{T|0}\left(\bx_T\right)\right]\\
    &-\frac{1}{2} \int_0^T \mathbb{E}_{\rho_{t|0}(\cdot)}\left[\beta_t\left\|\mathbf{s}_t\right\|_2^2+2\nabla \cdot\left(\beta_t \mathbf{s}_t - \bbf_t \right)\right] \dd t. \notag
\end{align*}
ISM is naturally connected to \citet{song2020sliced}, 
which supports flexible marginals and nonlinear forward processes but becomes significantly less scalable compared to ESM.

\subsection{Schr\"odinger Bridge}

The dynamic Schr\"odinger bridge aims to solve a full bridge 
\begin{equation}
    \inf_{\mathbb{P}\in\mathcal{D}(\rho_{\text{data}}, \rho_{\text{prior}})}\text{KL}(\mathbb{P}|\mathbb{Q}),\label{dynamic_SBP}
\end{equation}
where $\mathcal{D}(\rho_{\text{data}}, \rho_{\text{prior}})$ is the family of path measures with marginals $\rho_{\text{data}}$ and $\rho_{\text{prior}}$ at $t=0$ and $t=T$, respectively; $\mathbb{Q}$ is the prior process driven by $\dd \bx_t = \bbf_t(\bx_t) \dd t + \sqrt{2\beta_t\varepsilon}\dd \mathbf{\overrightarrow \bw}_t$. It also yields a stochastic control formulation \citep{Chen21, Pavon_CPAM_21, Caluya21}. 
\begin{align} 
    &\inf_{\bu\in \mathcal{U}} \E\bigg\{\int_0^T \frac{1}{2}\|\bu_t(\overrightarrow\bx_t)\|^2_2 \dd t \bigg\} \notag\\
    \text{s.t.} &\ \footnotesize{\dd  \overrightarrow\bx_t=\left[\bbf_t(\overrightarrow\bx)+\sqrt{\beta_t}\bu_t(\overrightarrow\bx)\right]\dd t+\sqrt{2\beta_t\varepsilon} \dd  \mathbf{\overrightarrow \bw}_t} \label{control_diffusion}\\
    &\ \ \footnotesize{\overrightarrow\bx_0\sim  \rho_{\text{data}} ,\ \  \overrightarrow\bx_T\sim  \rho_{\text{prior}}} \notag,
\end{align}
where $\mathcal{U}$ is the family of controls. The expectation is taken w.r.t $\overrightarrow{\rho}_t(\cdot)$, which denotes the PDF of the controlled diffusion \eqref{control_diffusion}; $\varepsilon$ is the temperature of the diffusion and the regularizer in EOT \citep{provably_schrodinger_bridge}.

Solving the underlying Hamilton–Jacobi–Bellman (HJB) equation and invoking the time reversal \citep{Anderson82} with $\varepsilon=\frac{1}{2}$, \emph{Schr\"{o}dinger system} yields the desired forward-backward stochastic differential equations (FB-SDEs) \citep{forward_backward_SDE}:
\begin{subequations}
\begin{align}
{\dd  \overrightarrow\bx_t}&=\footnotesize{\left[\bbf_t(\overrightarrow\bx_t) +  \beta_t\nabla\log\overrightarrow\psi_t(\overrightarrow\bx_t)\right]\dd t+ \sqrt{\beta_t} \dd  \overrightarrow\bw_t}, \label{FB-SDE-f}\\
{\dd  \overleftarrow\bx_t}&=\footnotesize{\left[\bbf_t(\overleftarrow\bx_t) - \beta_t \nabla\log\overleftarrow\varphi_t(\overleftarrow\bx_t)\right]\dd t+  \sqrt{\beta_t} \dd  {\overleftarrow\bw}_t}, \label{FB-SDE-b}
\end{align}\label{FB-SDE}
\end{subequations}
where $\overrightarrow\psi_t(\cdot) \overleftarrow\varphi_t(\cdot) =\overrightarrow{\rho}_t(\cdot)$, $\rho_0(\cdot)\sim \rho_{\text{data}}, \ \rho_T(\cdot)\sim \rho_{\text{prior}}$.

To solve the optimal controls (scores) $(\nabla\log\overrightarrow\psi, \nabla\log\overleftarrow\varphi)$, a standard tool is to leverage the nonlinear Feynman-Kac formula \citep{Ma_FB_SDE, Karatzas_Shreve, forward_backward_SDE} to learn a stochastic representation.
\begin{proposition}[Nonlinear Feynman-Kac representation] Assume Lipschitz smoothness and linear growth condition on the drift $\bbf$ and diffusion $g$ in the FB-SDE \eqref{FB-SDE}. Define $\overrightarrow y_t = \log \overrightarrow\psi_t(\bx_t)$ and $\overleftarrow y_t=\log \overleftarrow\varphi_t(\bx_t)$.
Then the stochastic representation follows 
\begin{align}
    \overleftarrow y_s&=\E\bigg[\overleftarrow y_T -\int_s^T {\Gamma_{\zeta}(\overleftarrow\bz_t; \overrightarrow\bz_t)}\dd t \bigg|\overrightarrow\bx_s=\textbf{x}_s\bigg],\notag\\
     \Gamma_{\zeta}(\overleftarrow\bz_t; \overrightarrow\bz_t)& \footnotesize{\equiv\frac{1}{2} \|  \overleftarrow \bz_t\|_2^2 + \frac{1}{2} \|  \overrightarrow \bz_t\|_2^2}  \label{Gamma_def} \\
     &\qquad + \footnotesize{\nabla \cdot \big(  \sqrt{\beta_t}\overleftarrow\bz_t - \bbf_t \big) + \zeta \langle  \overleftarrow \bz_t,  \overrightarrow \bz_t\rangle} \notag,
\end{align}
where ${\overrightarrow\bz_t =\sqrt{\beta_t} \nabla \overrightarrow y_t}$, ${\overleftarrow \bz_t =\sqrt{\beta_t} \nabla \overleftarrow y_t}$, and $\zeta=1$.
\end{proposition}

\vspace{-0.02in}
\section{Variational Schr\"odinger Diffusion Models}

SB outperforms SGMs in the theoretical potential of optimal transport and an intractable score function $\nabla\log\overrightarrow\psi_t(\bx_t)$ is exploited in the forward SDE for more efficient transportation plans. However, there is no free lunch in achieving such efficiency, and it comes with three notable downsides:
\begin{itemize}
    \item {Solving $\nabla\log\overrightarrow\psi_t$ in Eq.\eqref{FB-SDE-f} for optimal transport is prohibitively costly and may not be necessary \citep{Youssef_sample_transport, Rectified_group}}.
    \item {The nonlinear diffusion no longer yields closed-form expression of $\overrightarrow\bx_t$ given $\overrightarrow\bx_0$ \citep{forward_backward_SDE}.}
    \item {The ISM loss is inevitable and the estimator suffers from a large variance issue \citep{Hutchinson89}.}
\end{itemize}

\subsection{Variational Inference via Linear Approximation}

FB-SDEs naturally connect to the alternating-projection solver based on the IPF (a.k.a. Sinkhorn) algorithm, boiling down the full bridge \eqref{dynamic_SBP} to a half-bridge solver \citep{Pavon_CPAM_21, DSB, SBP_max_llk}. With $\mathbb{P}_1$ given and $k=1,2,...$, we have: 
\begin{subequations}
\begin{align}
    \mathbb{P}_{2k}&:=\argmin_{\mathbb{P}\in \mathcal{D}(\rho_{\text{data}},\  \cdot)} \text{KL}(\mathbb{P}\|\mathbb{P}_{2k-1})\label{half_bridge_b},\\
    \mathbb{P}_{2k+1}&:=\argmin_{\mathbb{P}\in \mathcal{D}(\cdot, \ \rho_{\text{prior}})} \text{KL}(\mathbb{P}\|\mathbb{P}_{2k})\label{half_bridge_f}.
\end{align}\label{half_bridge}
\end{subequations}
More specifically, \citet{forward_backward_SDE} proposed a neural network parameterization to model $(\overleftarrow\bz_t, \overrightarrow\bz_t)$ using $(\overleftarrow\bz^{\theta}_t, \overrightarrow\bz^{\omega}_t)$, where $\theta$ and $\omega$ refer to the model parameters, respectively. Each stage of the half-bridge solver proposes to solve the models alternatingly as follows 
\begin{subequations}
\begin{align}
    \overleftarrow{\mathcal{L}}(\theta)&=\small{-\int_0^T \E_{\overrightarrow\bx_t\backsim \eqref{FB-SDE-f}}\bigg[\Gamma_1(\overleftarrow\bz^{\theta}_t; \overrightarrow\bz^{\omega}_t)\dd t  \bigg|\overrightarrow\bx_0=\textbf{x}_0\bigg]}\label{SB-loss-b}\\
    \overrightarrow{\mathcal{L}}(\omega)&=\small{-\int_0^T \E_{\overleftarrow\bx_t\backsim \eqref{FB-SDE-b}}\bigg[\Gamma_1(\overrightarrow\bz^{\omega}_t; \overleftarrow\bz^{\theta}_t)\dd t \bigg|\overleftarrow\bx_T=\textbf{x}_T\bigg]}\label{SB-loss-f},
\end{align}\label{SB-loss}
\end{subequations}
where $\Gamma_1$ is defined in Eq.\eqref{Gamma_def} and $\backsim$ denotes the approximate simulation parametrized by neural networks \footnote{$\sim$ (resp. $\backsim$) denotes the exact (resp. parametrized) simulation.}

However, solving the backward score in Eq.\eqref{SB-loss-b} through simulations, akin to the ISM loss, is computationally demanding and affects the scalability in generative models.

To motivate simulation-free property, we leverage variational inference \citep{blei_VI} and study a linear approximation of the forward score $\nabla\log\overrightarrow\psi(\bx, t)\approx\bA_t \bx$ with $\bbf_t(\overrightarrow\bx_t)\equiv -\frac{1}{2}\beta_t\overrightarrow\bx_t$, which ends up with the variational FB-SDE (VFB-SDE):
\begin{subequations}
\begin{align}
\dd \overrightarrow\bx_t&=\left[-\frac{1}{2}\beta_t\overrightarrow\bx_t +  {\beta_t}\textcolor{black}{\bA_t\overrightarrow\bx_t }\right]\dd t+\sqrt{\beta_t} \dd \overrightarrow\bw_t, \label{linear_forward}\\
\dd \overleftarrow\bx_t&=\left[-\frac{1}{2}\beta_t\overleftarrow\bx_t - {\beta_t} \nabla \log  \overrightarrow{\rho}_{t}(\overleftarrow\bx_t)\right]\dd t+\sqrt{\beta_t} \dd \overleftarrow\bw_t\label{multi_backward},
\end{align}\label{FB-SDE_approx}
\end{subequations}
where $t\in[0, T]$ and $\nabla \log  \overrightarrow{\rho}_{t}$ is the score function of \eqref{linear_forward} and the conditional version is to be derived in Eq.\eqref{score_expression}.

The half-bridge solver is restricted to a class of OU processes $\text{OU}(\cdot, \rho_{\text{prior}})$ with the initial marginal $\rho_{\text{data}}$. 
\begin{equation*}
    \argmin_{\mathbb{P}\in \mathcal{D}(\cdot, \rho_{\text{prior}})} \text{KL}(\mathbb{P}\|\mathbb{P}_{2k})\Rightarrow \argmin_{\widehat{\mathbb{P}}\in \text{OU}(\cdot, \rho_{\text{prior}})} \text{KL}(\widehat{\mathbb{P}}\|\mathbb{P}_{2k}).
\end{equation*}
By the \emph{mode-seeking} property of the exclusive (reverse) KL divergence \citep{FKL_vs_BKL}, 
we can expect the optimizer $\widehat{\mathbb{P}}$ to be a \emph{local estimator} of the nonlinear solution in \eqref{half_bridge_b}. 

Additionally, the loss function \eqref{SB-loss-f} to learn the variational score $\bA_t$, where $t\in[0, T]$, can be simplified to
\begin{align}
    \overrightarrow{\mathcal{L}}(\bA)
    &=-\int_0^T \E_{\bx_t\backsim \eqref{multi_backward}} \bigg[\Gamma_{\zeta}(\bA_t\bx_t; \overleftarrow\bz^{\theta}_t)\dd t \bigg|\overleftarrow\bx_T=\textbf{x}_T\bigg]\label{VSDMf},
\end{align}
where $\Gamma_{\zeta}$ is defined in Eq.\eqref{Gamma_def}. Since the structure property $\overrightarrow\psi_t \overleftarrow\varphi_t =\overrightarrow{\rho}_t$ in Eq.\eqref{FB-SDE} is compromised by the variational inference, we propose to tune $\zeta$ in our experiments.

\subsection{Closed-form Expression of Backward Score}
\label{closed_form_score}
Assume a prior knowledge of $\bA_t$ is given, we can rewrite the forward process \eqref{linear_forward} in the VFB-SDE and derive a multivariate forward diffusion \citep{multivariateDM}:
\begin{equation}
\begin{split}
\label{mSGM-forward}
    \dd \overrightarrow\bx_t&=\left[-\frac{1}{2}\beta_t\bI +  {\beta_t}  \textcolor{black}{\bA_t }\right] \overrightarrow\bx_t \dd t+\sqrt{\beta_t} \dd \overrightarrow\bw_t \\
    &=-\frac{1}{2}\bD_t\beta_t \overrightarrow\bx_t \dd t+\sqrt{\beta_t} \dd \overrightarrow\bw_t,
\end{split}
\end{equation}
where $\bD_t=\bI-2\bA_t\in\mathbb{R}^{d\times d}$ is a positive-definite matrix \footnote{$\bD_t=-2\bA_t\in\mathbb{R}^{d\times d}$ when the forward SDE is VE-SDE.}. 
Consider the multivariate OU process \eqref{mSGM-forward}. The mean and covariance follow
\begin{subequations}
\begin{align}
    &\frac{\dd \bmu_{t|0}}{\dd t}=-\frac{1}{2}\beta_t \bD_t \bmu_{t|0}\label{mu_diffusion}\\
    &\frac{\dd \bSigma_{t|0}}{\dd t}=-\frac{1}{2}\beta_t \big(\bD_t  \bSigma_{t|0} +\bSigma_{t|0}\bD_t^{\intercal}\big)  + \beta_t\bI \label{sigma_diffusion}.
\end{align}
\end{subequations}
Solving the differential equations with the help of integration factors, the mean process follows
\begin{align}
    \bmu_{t|0}&=e^{-\frac{1}{2} [\beta\bD]_t} \bx_0,\label{mean_dyn}
\end{align}
where $[\beta\bD]_t=\int_0^t \beta_s \bD_s\dd s$. By matrix decomposition $\bSigma_{t|0}=\bC_t \bH_t^{-1}$ \citep{applied_sde}, the covariance process follows that:
    \begin{align}\label{cov_dynamics}
      \begin{pmatrix}
        \bC_t \\
        \bH_t
      \end{pmatrix} &=
      \exp\Bigg[
      \begin{pmatrix}
        -\frac{1}{2}[\beta\bD]_t & [\beta\bI]_t \\
        \bm{0} &  \frac{1}{2}[\beta\bD^{\intercal}]_t
      \end{pmatrix}
      \Bigg]
            \begin{pmatrix}
              {\bSigma}_0 \\
              {\bI} 
            \end{pmatrix},            
    \end{align}
where the above matrix exponential can be easily computed through modern computing libraries. Further, to avoid computing the expensive matrix exponential for high-dimensional problems, we can adopt a diagonal and time-invariant $\bD_t$.



Suppose $\bSigma_{t|0}$ has the Cholesky decomposition $\bSigma_{t|0} = \bL_t \bL_t^{\intercal}$ for some lower-triangular matrix $\bL_t$. We can have a closed-form update that resembles the SGM. 
\begin{align*}
    \overrightarrow\bx_t = \bmu_{t|0} + \bL_t\bepsilon,
\end{align*}
where $\bmu_{t|0}$ is defined in Eq.\eqref{mean_dyn} and $\bepsilon$ is the standard $d$-dimensional Gaussian vector. The score function follows
\begin{align}
     \nabla \log  \overrightarrow{\rho}_{t|0}(\overrightarrow\bx_t) &= -\frac12 \nabla[(\overrightarrow\bx_t - \bmu_t)^{\intercal} \bSigma_{t|0}^{-1} (\overrightarrow\bx_t - \bmu_t)] \notag\\
    &= -\bSigma_{t|0}^{-1}(\overrightarrow\bx_t-\bmu_t) \label{score_expression}\\
    &= -\bL_t^{-\intercal} \bL_t^{-1}\bL_t\bepsilon:=-\bL_t^{-\intercal}\bepsilon.\notag
\end{align}
Invoking the ESM loss function in Eq.\eqref{esm_loss}, we can learn the score function for Eq.\eqref{mSGM-forward}  
using a  neural network parametrization $s_{t}(\cdot)$ and optimize the loss function: 
\begin{equation}\label{MDM_loss}
    \mathbb{E}_t \big[\mathbb{E}_{\overrightarrow\bx_0} \mathbb{E}_{\overrightarrow\bx_t|\overrightarrow\bx_0}[\|-\bL_t^{-\intercal}\bm{\epsilon}-s_{t}(\bx_t)\|_2^2]].
\end{equation}

One may further consider preconditioning techniques \citep{EDM} or variance reduction \citep{multivariateDM} to stabilize training and accelerate training speed.

\paragraph{Speed-ups via Time-invariant and Diagonal $\bD_t$} If we parametrize $\bD_t$ as a time-invariant and diagonal positive-definite matrix, the formula \eqref{cov_dynamics} has simpler explicit expressions that do not require calling matrix exponential operators. We present such a result in Corollary \ref{corr:diagonal}. For the image generation experiment in Section \ref{image_gen}, we use such a diagonal parametrization when implementing the VSDM. 

\begin{corollary}\label{corr:diagonal}
If $\bD_t =\bm{\Lambda} \coloneqq \operatorname{diag}(\bm{\lambda})$, where $\lambda_i \geq 0,~ \forall 1 \leq i \leq d$. If we denote the $\sigma_t^2 \coloneqq \int_0^t \beta_s \mathrm{d}s$, then matrices $\bC_t$ and $\bH_t$ has simpler expressions with 
\begin{align*}
    & \bC_t = \bm{\Lambda}^{-1} \big\{ \exp(\frac{1}{2}\sigma_t^2 \bm{\Lambda}) - \exp(-\frac{1}{2}\sigma_t^2 \bm{\Lambda}) \big\} \\
    & \bH_t = \exp(\frac{1}{2}\sigma_t^2 \bm{\Lambda}),
\end{align*}
which leads to $\bC_t\bH_t^{-1} = \bm{\Lambda}^{-1} \big\{ \mathbf{I} - \exp (-\sigma_t^2 \bm{\Lambda})\big\}$. As a result, the corresponding forward transition writes 
\begin{align*}
    \bm{\mu}_{t|0} = \exp (-\frac{1}{2}\sigma_t^2 \bm{\Lambda})\bx_0,~ \bL_t = \bm{\Lambda}^{-\frac{1}{2}}\sqrt{\mathbf{I} - \exp(-\sigma_t^2 \bm{\Lambda}) }.
\end{align*}
\end{corollary}
In Corrolary \ref{corr:diagonal} detailed in Appendix \ref{diagonal_matrix}, since the matrix $\bm{\Lambda} = \operatorname{diag}(\bm{\lambda})$ is diagonal and time-invariant, the matrix exponential and square root can be directly calculated element-wise on each diagonal elements $\lambda_i$ independently.

\subsubsection{Backward SDE}
\label{backward_DSM_setup}
Taking the time reversal \citep{Anderson82} of the forward multivariate OU process \eqref{mSGM-forward}, the backward SDE satisfies
\begin{align}
\dd \overleftarrow\bx_t=(-\frac{1}{2}\bD_t\beta_t \overleftarrow\bx_t - {\beta_t} s_t(\overleftarrow\bx_t))\dd t+ \sqrt{\beta_t} \dd \overleftarrow\bw_t. \label{backward_DSM}
\end{align}
Notably, with a general PD matrix $\bD_t$, the prior distribution follows that $\bx_T \sim \mathrm{N}(\bm{0}, \bSigma_{T|0})$\footnote{See the Remark on the selection of $\rho_{\text{prior}}$ in section \ref{part_1}.}. We also note that the prior is now limited to Gaussian distributions, which is not a general bridge anymore.

\subsubsection{Probability Flow ODE}

We can follow \citet{score_sde} and obtain the deterministic process directly:
\begin{align}
\dd \overleftarrow\bx_t&=\bigg(-\frac{1}{2}\bD_t\beta_t \overleftarrow\bx_t - \frac{1}{2}{\beta_t} s_t(\overleftarrow\bx_t)\bigg)\dd t, \label{backward_DSM_ode}
\end{align}
where $\bx_T \sim \mathrm{N}(\bm{0}, \bSigma_{T|0})$ and the sample trajectories follow the same marginal densities $\overrightarrow{\rho}_{t}(\bx_t)$ as in the SDE.

\begin{algorithm*}[tb]
   \caption{Variational Schr\"odinger Diffusion Models (VSDM). $\rho_{\text{prior}}$ is fixed to a Gaussian distribution. $\eta_k$ is the step size for SA and $h$ is the learning rate for the backward sampling of Eq.\eqref{backward_DSM}. The exponential moving averaging (EMA) technique can be used to further stabilize the algorithm.}
   \label{VSDM_alg}
\begin{algorithmic}
\REPEAT
   \STATE{\textbf{Simulation-free Optimization of Backward Score}}
    \STATE{\text{Draw $\bx_0\sim\rho_{\text{data}}$, $n\sim \{0, 1, \cdots, N-1\}$, $\bm{\epsilon}\sim\mathrm{N}(\bm{0}, \bI)$.}}
    \STATE{\text{Sample $\bx_{nh}|\bx_0\sim \mathrm{N}(\bmu_{nh|0}, \bSigma_{nh|0})$ by Eq.\eqref{mean_dyn} and \eqref{cov_dynamics} given $\bA_{nh}^{(k)}$.}}
    \STATE{\text{Cache $\{\bmu_{nh|0}\}_{n=0}^{N-1}$ and $\{\bL_{nh}^{-\intercal}\}_{n=0}^{N-1}$ via Cholesky decomposition of $\{\bSigma_{nh}\}_{n=0}^{N-1}$ to avoid repeated computations.}}
    \STATE{\text{Optimize the score functions
    $s^{(k+1)}_{nh}$ sufficiently through the loss $\mathbb{E}_n \big[\mathbb{E}_{\bx_0}\mathbb{E}_{\bx_{nh}|\bx_0}[\|-\bL_{nh}^{-\intercal}\bm{\epsilon}-s^{(k+1)}_{nh}(\bx_{nh})\|_2^2]]$.}}
   \STATE{\textbf{Optimization of Variational Score via Stochastic Approximation (SA)}}
   \STATE{\text{Simulate the backward trajectory $\overleftarrow\bx^{(k+1)}_{nh}$ given $\bA_{nh}^{(k)}$ and $s^{(k+1)}_{nh}$ via Eq.\eqref{backward_DSM_discrete}, where $\overleftarrow\bx^{(k+1)}_{(N-1)h}\sim \mathrm{N}(\bm{0}, \bSigma^{(k)}_{(N-1)h|0})$.}}
   \STATE{\text{Optimize variational score $\bA_{nh}^{(k+1)}$ using the loss function \eqref{VSDMf}, where $n\in\{0, 1, \cdots, N-1\}$:}}
   \begin{equation}\label{dsa_iterate_main}
       \bA_{nh}^{(k+1)}=\bA_{nh}^{(k)}-\eta_{k+1} \nabla\overrightarrow{\mathcal{L}}_{nh}(\bA_{nh}^{(k)}; \overleftarrow\bx^{(k+1)}_{nh}).
   \end{equation}
   \UNTIL{\text{Stage} $k=k_{\max}$}
   \STATE{\text{Sample $\overleftarrow\bx_{0}$ with stochastic (resp. deterministic) trajectories via the discretized Eq.\eqref{backward_DSM}} (resp. Eq.\eqref{backward_DSM_ode}).}
    \vskip -3 in
\end{algorithmic}
\end{algorithm*}

\subsection{Adaptive Diffusion via Stochastic Approximation}

Our major goal is to generate high-fidelity data with efficient transportation plans based on the optimal $\bA_t^{\star}$ in the forward process \eqref{mSGM-forward}. 
However, the optimal $\bA_t^{\star}$ is not known \emph{a priori}. To tackle this issue, we leverage stochastic approximation (SA) \citep{RobbinsM1951, Albert90} to adaptively optimize the variational score $\bA_t^{(k)}$ through optimal transport and simulate the backward trajectories. 
\begin{itemize}
\item[(1)] Simulate backward trajectories $\{\overleftarrow\bx^{(k+1)}_{nh}\}_{n=0}^{N-1}$ via the Euler–Maruyama (EM) scheme of the backward process \eqref{backward_DSM} with a learning rate $h$.
\item[(2)] Optimize variational scores $\big\{\bA^{(k)}_{nh}\}_{n=0}^{N-1}$: 
$$\bA_{nh}^{(k+1)}=\bA_{nh}^{(k)}-\eta_{k+1} \nabla\overrightarrow{\mathcal{L}}_{nh}(\bA_{nh}^{(k)}; \overleftarrow\bx^{(k+1)}_{nh}),$$
\end{itemize}
where $\nabla\overrightarrow{\mathcal{L}}_{nh}(\bA_{nh}^{(k)}; \overleftarrow\bx^{(k+1)}_{nh})$ is the gradient of the loss function \eqref{VSDMf} at time $nh$ and is known as the random field. We expect that the simulation of backward trajectories $\{\overleftarrow\bx^{(k+1)}_{nh}\}_{n=0}^{N-1}$ given $s^{(k+1)}_{nh}$ helps the optimization of $\bA_{nh}^{(k+1)}$ and the optimized $\bA_{nh}^{(k+1)}$ in turn contributes to a more efficient transportation plan for estimating $s^{(k+2)}_{nh}$ and simulating the backward trajectories $\{\overleftarrow\bx^{(k+2)}_{nh}\}_{n=0}^{N-1}$.

\paragraph{Trajectory Averaging} The stochastic approximation algorithm is a standard framework to study adaptive sampling algorithms \citep{Liang07}. 
Moreover, the formulation suggests to stabilize the trajectories \citep{Polyak} with averaged parameters $\overline{\bA}_{nh}^{(k)}$ as follows
\begin{equation*}
    \overline{\bA}_{nh}^{(k)}=\sum_{i=1}^{k} \bA_{nh}^{(i)}=\bigg(1-\frac{1}{k}\bigg)\overline{\bA}_{nh}^{(k-1)} + \frac{1}{k}  \bA_{nh}^{(k)},
\end{equation*}
where $\overline{\bA}_{nh}^{(k)}$ is known to be an asymptotically efficient (optimal) estimator \citep{Polyak} in the local state space $\mathcal{A}$ by assumption \ref{ass_local_state_space}. 

\paragraph{Exponential Moving Average (EMA)} Despite guarantees in convex scenarios, the parameter space differs tremendously in different surfaces in non-convex state space $\mathcal{A}$. Empirically, if we want to exploit information from multiple modes, a standard extension is to employ the EMA technique \citep{slides_SA}:
\begin{equation*}
    \overline{\bA}_{nh}^{(k)}=(1-\eta)\overline{\bA}_{nh}^{(k-1)} + \eta  \bA_{nh}^{(k)}, \text{where } \eta \in(0, 1).
\end{equation*}
The EMA techniques are widely used empirically in diffusion models and Schr\"odinger bridge \citep{Song_improved_techniques_20, DSB, forward_backward_SDE} to avoid oscillating trajectories. 
Now we are ready to present our methodology in Algorithm \ref{VSDM_alg}.

\paragraph{Computational Cost} Regarding the wall-clock computational time: i) training (\textcolor{dark2blue}{linear}) variational scores, albeit in a \emph{simulation-based manner}, becomes significantly faster than estimating \textcolor{darkred}{nonlinear} forward scores in Schr\"odinger bridge; ii) the variational parametrization greatly reduced the number of model parameters, which yields a much-reduced variance in the Hutchinson's estimator \citep{Hutchinson89}; iii) since we don't need to update $\bA_t$ as often as the backward score model, we can further amortize the training of $\bA_t$. In the simulation example in Figure.\ref{eigenvalue_analysis_main}(b), VSDM is only 10\% slower than the SGM with the same training complexity of backward scores while still maintaining efficient convergence of variational scores.



\section{Convergence of Stochastic Approximation}

In this section, we study the convergence of $\bA_{t}^{(k)}$ to the optimal $\bA_{t}^{\star}$, where $t\in[0,T]$ \footnote{We slightly abuse the notation and generalize $\bA_{nh}^{(k)}$ to $\bA_{t}^{(k)}$.}. The primary objective is to show  the iterates \eqref{dsa_iterate_main} follow
the trajectories of the dynamical system asymptotically:
\begin{equation}
\label{mean_field_perturbed_ode_main}
    \dd \bA_t = \nabla \overrightarrow{\bL}_t(\bA_t)\dd s,
\end{equation}
where $\frac{\dd \bA_t}{\dd s}=\lim_{\eta\rightarrow 0} \frac{\bA_t^{(k+1)}-\bA_t^{(k)}}{\eta}$ and $\nabla  \overrightarrow{\bL}_t(\cdot)$ is the mean field at time $t$:
\begin{equation}
\begin{split}
\label{mean_field_perturbed_main}
\nabla  \overrightarrow{\bL}_t(\bA_t)&=\int_{\MX} \nabla\overrightarrow{\mathcal{L}}_t(\bA_t; \overleftarrow\bx^{(\cdot) }_t)\overleftarrow{\rho}_{t}(\dd\overleftarrow\bx^{(\cdot) }_t),\\
\end{split}
\end{equation}
where $\MX$ denotes the state space of data $\bx$ and $\nabla\overrightarrow{\mathcal{L}}_t$ denotes the gradient w.r.t. $\bA_t$; $\overleftarrow{\rho}_{t}$ is the distribution of the continuous-time interpolation of the discretized backward SDE \eqref{backward_DSM_discrete} from $t=T$ to $0$. We denote by $ \bA_t^{\star}$ one of the solutions of $\nabla  \overrightarrow{\bL}_t( \bA_t^{\star})=\bm{0}$. 

The aim is to find the optimal solution $\bA_{t}^{\star}$ to the mean field $\nabla\overrightarrow{\bL}_t(\bA_{t}^{\star})=\bm{0}$. However, we acknowledge that the equilibrium is not unique in general nonlinear dynamical systems. To tackle this issue, we focus our analysis around a neighborhood $\bTheta$ of the equilibrium by assumption \ref{ass_local_state_space}. After running sufficient many iterations with a small enough step size $\eta_k$, suppose $\bA_t^{(k)}\in\bTheta$ is somewhere near one equilibrium $\bA_t^{\star}$ (out of all equilibrium), then by the induction method, the iteration tends to get trapped in the same region as shown in Eq.\eqref{final_sa_convergence2} and yields the convergence to one equilibrium $\bA_t^{\star}$.  We also present the variational gap of the (sub)-optimal transport and show our transport is more efficient than diffusion models with Gaussian marginals.

Next, we outline informal assumptions and sketch our main results, reserving formal ones for readers interested in the details in the appendix. We also formulate the optimization of the variational score $\bA_{t}$ using stochastic approximation in Algorithm \ref{alg:SA_algorithm} in the supplementary material. 

\paragraph{Assumption A1}(Regularity). \emph{(Positive definiteness) For any $t\geq 0$ and $\bA_t \in \mathcal{A}$, $\bD_t=\bI - 2\bA_t$ is positive definite. (Locally strong convexity) For any stable local minimum $\bA_t^{\star}$ with $\nabla \overrightarrow{\bL}_t(\bA_t^{\star})= \bm{0}$, there is always a neighborhood $\bTheta$ s.t. $\bA_t^{\star}\in \bTheta\subset \mathcal{A}$ and $\overrightarrow{\bL}_t$ is strongly convex in $\bTheta$.}

By the mode-seeking property of the exclusive (reverse) KL divergence \citep{FKL_vs_BKL}, we only make a mild assumption on a small neighborhood of the solution and expect the convergence given proper regularities.

\paragraph{Assumption A2}(Lipschitz Score). \emph{For any $t\in[0, T]$, the score $\nabla\log  \overrightarrow\rho_{t}$ is $L$-Lipschitz.} 

\paragraph{Assumption A3}(Second Moment Bound). \emph{The data distribution has a bounded second moment.}

\paragraph{Assumption A4}(Score Estimation Error). \emph{We have bounded score estimation errors in $L^2$ quantified by $\epsilon_{\text{score}}$.}

We first use the multivariate diffusion to train our score estimators  $\{s^{(k)}_{t}\}_{n=0}^{N-1}$ via the loss function \eqref{MDM_loss} based on the pre-specified $\bA_t^{(k)}$ at step $k$. Similar in spirit to  \citet{chen2023improved, Sitan_22_sampling_is_easy}, we can show the generated samples based on $\{s^{(k)}_{t}\}_{n=0}^{N-1}$ are close in distribution to the ideal samples in Theorem  \ref{theorem:quality_of_data}. The novelty lies in the extension of \textcolor{dark2blue}{single-variate diffusions} to \textcolor{dark2blue}{multi-variate diffusions}.

\paragraph{Theorem 1}(Generation quality, informal). \emph{Assume assumptions \ref{ass_local_state_space}-\ref{ass_score_estimation} hold with a fixed $\bA^{(k)}_t$, the generated data distribution is close to the data distributions $\rho_{\text{data}}$ such that
    \begin{equation*}
    \begin{split}
        \mathrm{TV}(\overleftarrow{\rho}^{(k)}_{0}, \rho_{\text{data}})&\lesssim  \exp(-T) + (\sqrt{dh} +\epsilon_{\text{score}})\sqrt{T}.
    \end{split}
    \end{equation*}
    }

To show the convergence of $\bA_{t}^{(k)}$ to $ \bA_{t}^{\star}$, the proof hinges on a stability condition such that the solution asymptotically tracks the equilibrium $ \bA_{t}^{\star}$ of the mean field \eqref{mean_field_perturbed_ode_main}.

\paragraph{Lemma 2}(Local stability, informal). \emph{Assume the assumptions \ref{ass_local_state_space} and \ref{ass_smoothness} hold. For $\forall t\in[0, T]$ and $\forall\bA\in\bTheta$, the solution satisfies a local stability condition such that
        \begin{equation*}
            \langle \bA- \bA_t^{\star}, \nabla \overrightarrow{\bL}_t(\bA) \rangle \gtrsim  \|\bA-\bA_t^{\star}\|_2^2.
        \end{equation*}}


The preceding result illustrates the convergence of the solution toward the equilibrium on average. The next assumption assumes a standard slow update of the SA process, which is standard for theoretical analysis but may not be always needed in empirical evaluations. 

\vspace{-0.04in}
\paragraph{Assumption A5}(Step size). \emph{The step size $\{\eta_{k}\}_{k\in \mathrm{N}}$ is a positive and decreasing sequence }
\begin{equation*} 
\eta_{k}\rightarrow 0, \ \ \sum_{k=1}^{\infty} \eta_{k}=+\infty,\ \sum_{k=1}^{\infty} \eta^2_{k}<+\infty.
\end{equation*}

Next, we use the stochastic approximation theory to prove the convergence of $\bA_t^{(k)}$ to an equilibrium $  \bA_t^{\star}$.

\paragraph{Theorem 2}(Convergence in $L^2$). \emph{Assume assumptions \ref{ass_local_state_space}-\ref{ass_step_size} hold. The variational score $\bA_t^{(k)}$ converges to an equilibrium $ \bA_t^{\star}$ in $L^2$ such that
    \begin{equation*}
    \E[\|\bA_t^{(k)}- \bA_t^{\star}\|_2^2]\leq 2 \eta_{k},
\end{equation*}
where the expectation is taken w.r.t samples from $\overleftarrow{\rho}_{t}^{(k)}$.
}

In the end, we adapt Theorem \ref{theorem:quality_of_data} again to show the adaptively generated samples are asymptotically close to the samples based on the optimal $\bA_{t}^{\star}$ in Theorem \ref{theorem_adaptive_sampling}, which quantifies the quality of data based on more efficient transportation plans.

\vspace{-0.05in}
\paragraph{Theorem 3}(Generation quality of adaptive samples). \emph{Given assumptions \ref{ass_local_state_space}-\ref{ass_step_size}, the generated sample distribution at stage $k$ is close to the exact sample distribution 
based on the equilibrium $ \bA_{t}^{\star}$ such that
\begin{equation*}
    \begin{split}
    \mathrm{TV}&(\overleftarrow{\rho}^{\star}_{0}, \rho_{\text{data}})\lesssim \exp(-T)+ (\sqrt{dh} + \epsilon_{\text{score}}+\sqrt{\eta_k})\sqrt{T}.
    \end{split}
    \end{equation*}}


\section{Variational Gap} 

Recall that the optimal and variational forward SDEs follow
\begin{align*}
    {\dd  \overrightarrow\bx_t}&={\left[\bbf_t(\overrightarrow\bx_t) +  \beta_t\nabla\log\overrightarrow\psi_t(\overrightarrow\bx_t)\right]\dd t}+ \sqrt{\beta_t} \dd  \overrightarrow\bw_t, \\
    \dd \overrightarrow\bx_t&=\left[\bbf_t(\overrightarrow\bx_t) +  {\beta_t} \bA_t^{(k)}\overrightarrow\bx_t\right]\dd t+\sqrt{\beta_t} \dd \overrightarrow\bw_t,\\
    \dd \overrightarrow\bx_t&=\left[\bbf_t(\overrightarrow\bx_t) +  {\beta_t} \bA_t^{\star}\overrightarrow\bx_t\right]\dd t+\sqrt{\beta_t} \dd \overrightarrow\bw_t,
\end{align*}
where we abuse the notion of $\overrightarrow\bx_t$ for the sake of clarity and they represent three different processes. Despite the improved efficiency based on the ideal $\bA_t^{\star}$ compared to the vanilla $\bA_t\equiv\bm{0}$, the variational score inevitably yields a sub-optimal transport in general nonlinear transport. We denote the law of the above processes by $\mathrm{L}$, $ {\mathrm{L}}^{(k)}$, and $\mathrm{L}^{\star}$. 
To assess the disparity, we leverage the Girsanov theorem to study the variational gap.

\paragraph{Theorem 4}(Variational gap). \emph{Assume the assumption \ref{ass_smoothness} and Novikov’s condition hold. Assume $\bbf_t$ and $\nabla\log\overrightarrow\psi_t$ are Lipschitz smooth and satisfy the linear growth. The variational gap follows that
\begin{equation*}
\begin{split}
    &\mathrm{KL}({\mathrm{L}}\|{\mathrm{L}}^{\star})=\frac{1}{2} \int_0^T \E\bigg[\beta_t \|\bA_t^{\star}\overrightarrow\bx_t - \nabla\log\overrightarrow\psi_t(\overrightarrow\bx_t)\|_2^2\bigg]\dd t\\
    &\mathrm{KL}({\mathrm{L}}\| {\mathrm{L}}^{(k)}) \lesssim \eta_k + \mathrm{KL}({\mathrm{L}}\|{\mathrm{L}}^{\star}).
\end{split}
\end{equation*}}

\paragraph{Connections to Gaussian Schr\"odinger bridge (GSB)}

When data follows a Gaussian distribution, VSDM approximates the closed-form OT solution of Schr\"odinger bridge \citep{EOT_gaussian, SB_closed_form}. We refer readers to Theorem 3  \citep{SB_closed_form} for the detailed transportation plans. Compared to the vanilla $\bA_t\equiv \bm{0}$, we can significantly reduce the variational gap with $\mathrm{KL}({\mathrm{L}}\|{\mathrm{L}}^{\star})$ using proper parametrization and sufficient training.

We briefly compare VSDM to SGM and SB in the following:
\begin{table}[h]
\vspace{-0.2in}
\begin{tabular}{ c | c | c | c }
\label{VSDM_SGM_SB}
\footnotesize
Properties & SGM & SB & VSDM \\ 
\hline
\footnotesize{Entropic Optimal Transport} & $\color{red}\times$  &  \footnotesize{\textcolor{darkblue}{Optimal}}  &  \footnotesize{\textcolor{darkblue}{Sub-Optimal}} \\ 
\hline
 {Simulation-free Forward} &  \color{green}\checkmark  &  $\color{red}\times$ &  \color{green}\checkmark   \\
\end{tabular}
\end{table}

\section{Empirical Studies}

\subsection{Comparison to Gaussian Schrodinger Bridge}

VSDM is approximating GSB \citep{SB_closed_form} when both marginals are Gaussian distributions. To evaluate the solutions, we run our VSDM with a fixed $\beta_t\equiv 4$ in Eq.(25) in \citet{score_sde} and use the same marginals to replicate the VPSDE of the Gaussian SB with $\alpha_t\equiv 0$ and $c_t\equiv -2$ in Eq.(7) in \citet{SB_closed_form}. We train VSDM with 20 stages and randomly pick 256 samples for presentation. We compare the flow trajectories from both models and observe in Figure \ref{vsdm_vs_dsm} that the ground truth solution forms an almost linear path, while our VSDM sample trajectories exhibit a consistent alignment with trajectories from Gaussian SB. We attribute the bias predominantly to score estimations and numerical discretization.

\begin{figure}[!ht]
  \centering
  \vspace{-0.1in}
  \subfigure[{GSB}]{\includegraphics[scale=0.33]{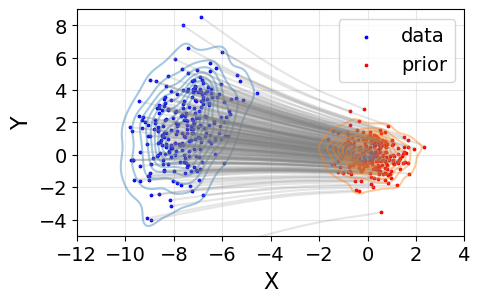}}
  \subfigure[{VSDM}]{\includegraphics[scale=0.33]{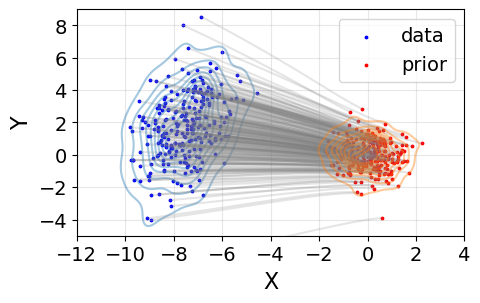}} 
  \vskip -0.1in
  \caption{Gaussian SB (GSB) v.s. VSDM on the flow trajectories.}\label{vsdm_vs_dsm}
  \vspace{-1em}
\end{figure}

\subsection{Synthetic Data}

We test our variational Schr\"odinger diffusion models ({VSDMs}) on two synthetic datasets: spiral and checkerboard (detailed in section \ref{Checkerboard_data}). We include SGMs as the baseline models and aim to show the strength of VSDMs on general shapes with straighter trajectories. As such, we stretch the Y-axis of the spiral data by 8 times and the X-axis of the checkerboard data by 6 times and denote them by spiral-8Y and checkerboard-6X, respectively.

We adopt a monotone increasing $\{\beta_{nh}\}_{n=0}^{N-1}$ similar to \citet{score_sde} and denote by $\beta_{\min}$ and $\beta_{\max}$ the minimum and maximum of $\{\beta_{nh}\}_{n=0}^{N-1}$. We fix $\zeta=0.75$ and $\beta_{\min}=0.1$ and we focus on the study with different $\beta_{\max}$. We find that SGMs work pretty well with $\beta_{\max}=10$ (SGM-10) on standard isotropic shapes. However, when it comes to spiral-8Y, the SGM-10 \textcolor{darkred}{\emph{struggles to recover}} the boundary regions on the spiral-8Y data as shown in Figure \ref{VSDM_SGM} (top).

\paragraph{Generations of Anisotropic Shapes} To illustrate the effectiveness of our approach, Figure \ref{VSDM_SGM} (bottom) shows that VSDM-10 accurately reconstructs the edges of the spiral and generates high-quality samples.

\begin{figure}[!ht]
  \centering
    \subfigure{\includegraphics[scale=0.13]{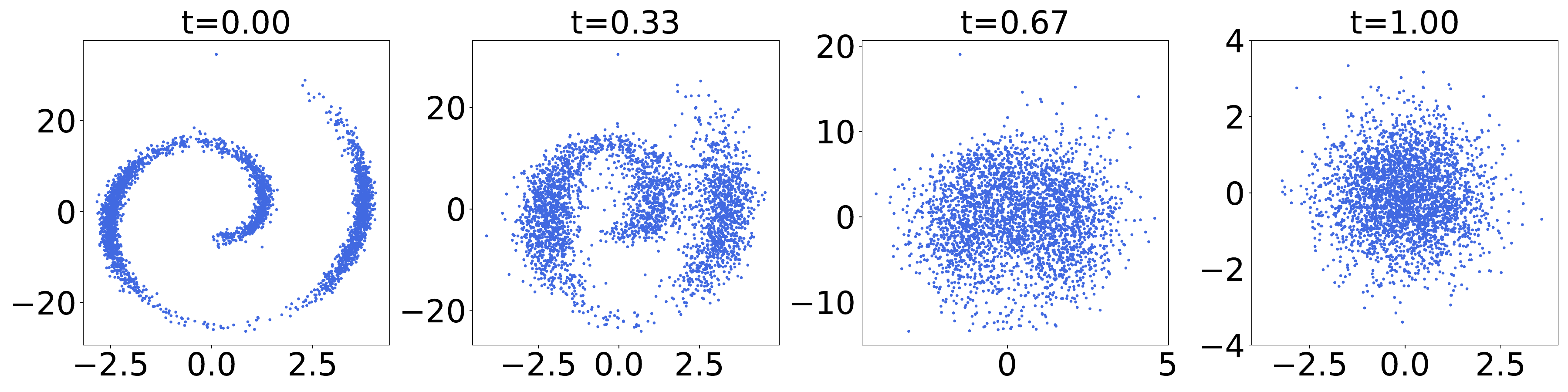}}
  \vspace{-0.1in}
    \subfigure{\includegraphics[scale=0.13]{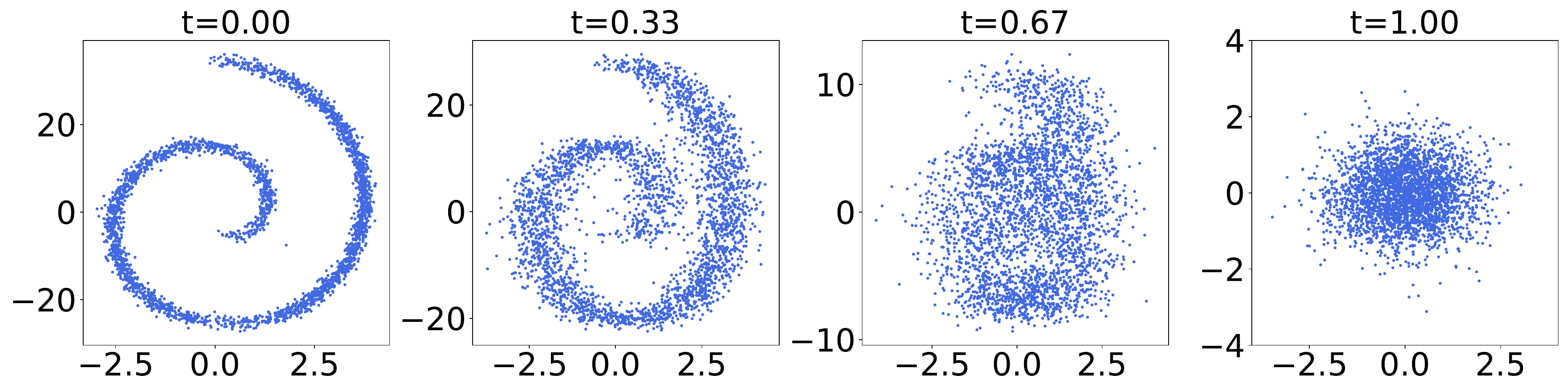}}
  \caption{Variational Schr\"odinger diffusion models (VSDMs, bottom) v.s. SGMs (top) with the same hyperparameters ($\beta_{\max}=10$).}\label{VSDM_SGM}
  \vspace{-1em}
\end{figure}

\paragraph{Straighter Trajectories} The SGM-10 fails to fully generate the anisotropic spiral-8Y and increasing $\beta_{\max}$ to 20 or 30 (SGM-20 and SGM-30) signiﬁcantly alleviates this issue. However, we observe that \textcolor{darkred}{excessive $\beta_{\max}$ values in SGMs compromise the straightness} and leads to inefficient transport, especially in the X-axis of spiral-8Y. 

\begin{figure}[!ht]
  \centering
    \vspace{-0.07in}
  \subfigure[\small{SGM-10}]{\includegraphics[scale=0.15]{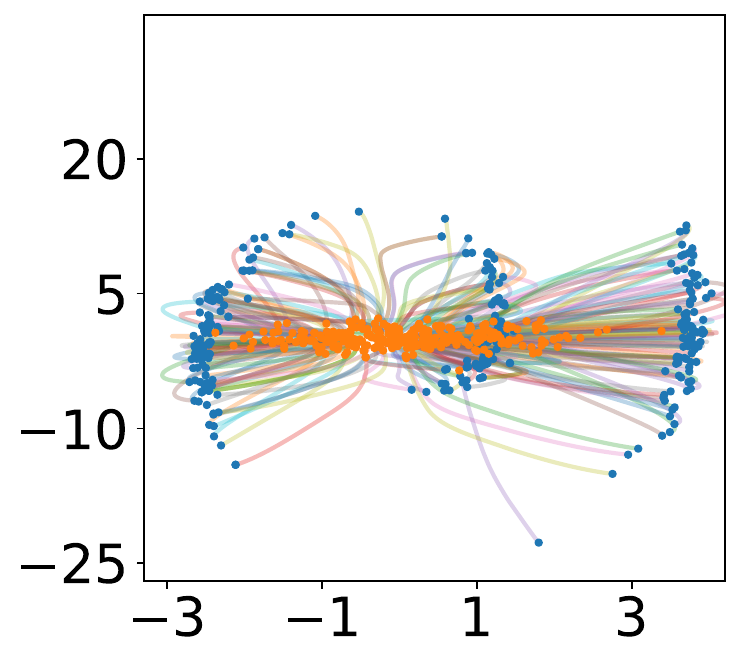}}
  \subfigure[\small{SGM-20}]{\includegraphics[scale=0.15]{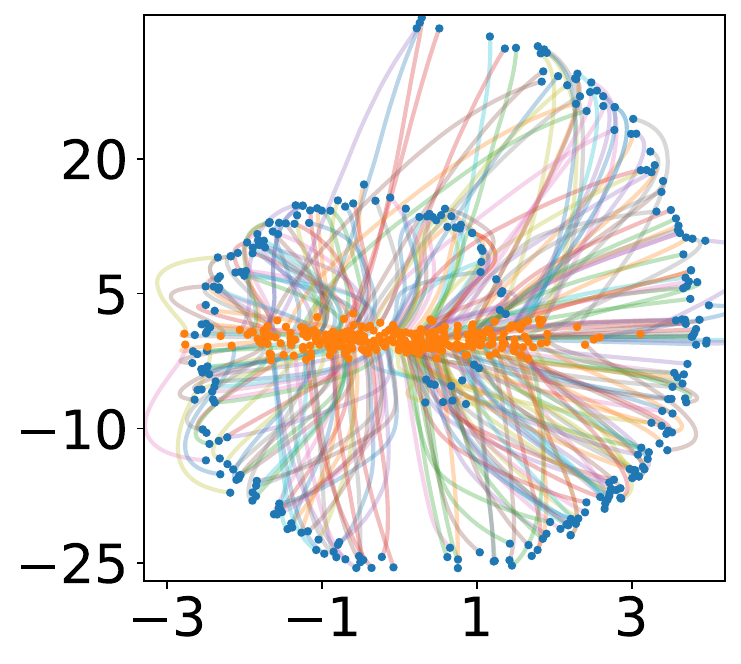}}
    \subfigure[\small{SGM-30}]{\includegraphics[scale=0.15]{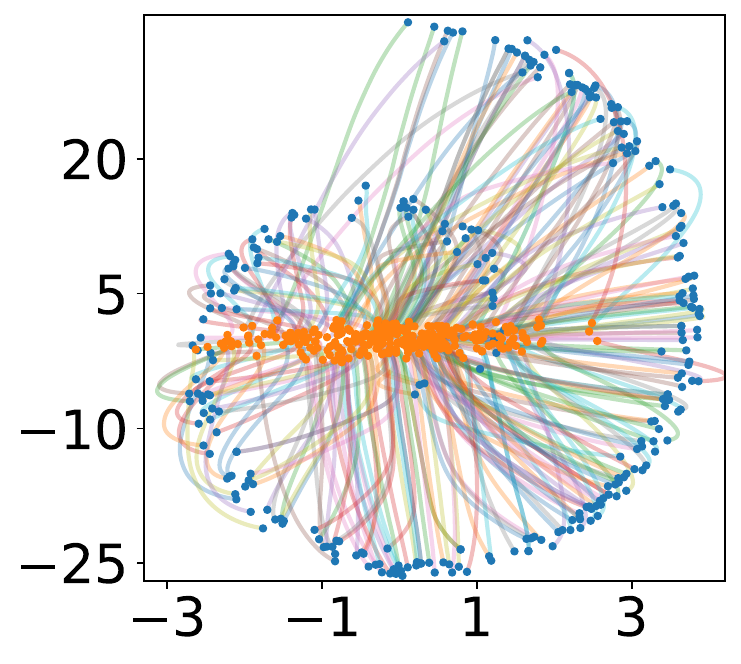}}
  \subfigure[\small{VSDM-10}]{\includegraphics[scale=0.15]{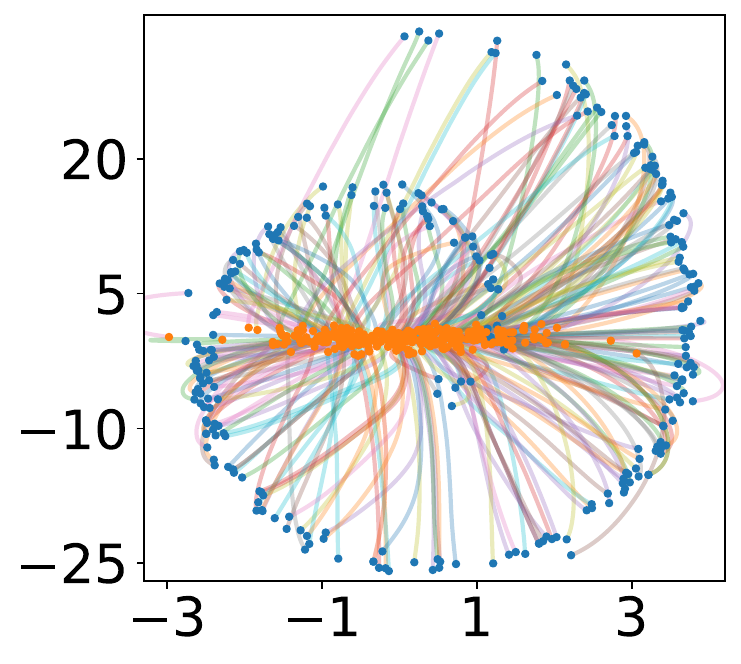}} \vskip -0.1in
  \caption{Probability flow ODE via VSDMs and SGMs. SGM with $\beta_{\max}=10$ is denoted by SGM-10 for convenience.}\label{trajectories_dynamics}
  \vspace{-0.5em}
\end{figure}

Instead of setting excessive $\beta_{\max}$ on both axes, our VSDM-10, by contrast, proposes \textcolor{darkblue}{conservative diffusion scales} on the X-axis of spiral-8Y and explores more on the Y-axis of spiral-8Y. As such, we obtain around \textcolor{dark2blue}{40\% improvement on the straightness} in Figure \ref{trajectories_dynamics} and Table \ref{straight_metric}.

Additional insights into a similar analysis of the checkboard dataset, convergence analysis, computational time, assessments of straightness, and evaluations via a smaller number of function evaluations (NFEs) can be found in Appendix \ref{syn_appendix}.

\subsection{Image Data Modeling}\label{image_gen}

\begin{figure}
\label{fig:c10_gen}
\centering
\includegraphics[width=0.3\textwidth]{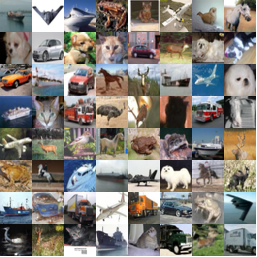}
\vspace{-2mm}
\caption{Unconditional generated samples from VSDM on CIFAR10 (32$\times 32$ resolution) trained from scratch.}
\vspace{-2mm}
\end{figure}

\begin{table*}[!t]
\vspace{-5pt} 
\small
\begin{sc}
\caption{Convergence speed of FID values for VSDM.} \label{tab:c10_converge}
\begin{center}
\begin{scriptsize}
\resizebox{0.95\textwidth}{!}{
\begin{tabular}{ccccccccccc}
\toprule
\text{K Images}
& 0 & 10k & 20k & 30k & 40k & 50k & 100k & 150k & 200k & \text{converge} \\
\midrule
\textbf{FID$\downarrow$ (NFE=35)} & 406.13 & 13.13 & 8.65 & 6.83 & 5.66 & 5.21 & 3.62 & 3.29 & 3.01 & 2.28 \\
\bottomrule
\end{tabular}}
\end{scriptsize}
\end{center}
\vspace{-5pt} 
\end{sc}
\end{table*}

\paragraph{Experiment Setup}
In this experiment, we evaluate the performance of VSDM on image modeling tasks. We choose the CIFAR10 dataset
as representative image data to demonstrate the scalability of the proposed VSDM on generative modeling of high-dimensional distributions. We refer to the code base of FB-SDE \citep{forward_backward_SDE} and use the same forward diffusion process of the EDM model \citep{EDM}. Since the training of VSDM is an alternative manner between forward and backward training, we build our implementations based on the open-source diffusion distillation code base \citep{luo2024diff} \footnote{\scriptsize{See code in \url{https://github.com/pkulwj1994/diff_instruct}}}, which provides a high-quality empirical implementation of alternative training with EDM model on CIFAR10 data. To make the VSDM algorithm stable, we simplify the matrix $\bD_t$ to be diagonal with learnable diagonal elements, which is the case as we introduced in Corollary \ref{corr:diagonal}. We train the VSDM model from scratch on two NVIDIA A100-80G GPUs for two days and generate images from the trained VSDM with the Euler–Maruyama numerical solver with 200 discretized steps for generation.

\vspace{-0.1in}
\begin{table}[h]
  \small
    \centering
    \begin{sc}
    \caption{
      CIFAR10 evaluation using sample quality (FID score). Our VSDM outperforms other optimal transport baselines by a large margin.
    }
      \begin{tabular}{llcccccc}
        \toprule
        {Class} & {Method}  & {FID $\downarrow$} \\
        \midrule
        \multirow{4}{*}{OT}
        & \textbf{VSDM (ours)}          &  \textbf{2.28} \\[1pt]
        & SB-FBSDE \citep{forward_backward_SDE}          &  3.01 \\[1pt]
        & DOT \citep{Tanaka2019DiscriminatorOT}  & 15.78 \\[1pt]
        & DGflow \citep{Ansari2020RefiningDG}    &  9.63 \\[1pt]
        \midrule
        \multirow{5}{*}{SGMs}
        & SDE (\citet{score_sde})         & 2.92  \\[1pt] 
        & ScoreFlow \citep{song_likelihood_training}             & 5.7  \\[1pt]
        & VDM \citep{Kingma2021VariationalDM}             & 4.00  \\[1pt]
        & LSGM\citep{vahdat2021score}                   & 2.10 \\[1pt]
        & EDM\citep{EDM}                   &\textbf{1.97} \\[1pt]
        \bottomrule
      \end{tabular} \label{table:c10_fid}
\end{sc}
\vskip -0.15in
\end{table}

\paragraph{Performances.}
We measure the generative performances in terms of the Fretchat Inception Score (FID \citep{Heusel2017GANsTB}, the lower the better), which is a widely used metric for evaluating generative modeling performances.

Tables \ref{table:c10_fid} summarize the FID values of VSDM along with other optimal-transport-based and score-based generative models on the CIFAR10 datasets (unconditional without labels).  The VSDM outperforms other optimal transport-based models with an \textcolor{dark2blue}{FID of 2.28}. This demonstrates that the VSDM has applicable scalability to model high-dimensional distributions. Figure \ref{fig:c10_gen} shows some non-cherry-picked unconditional generated samples from VSDM trained on the CIFAR10 dataset.

\paragraph{Convergence Speed.}
To demonstrate the convergence speed of VSDM along training processes, we record the FID values in Table \ref{tab:c10_converge} for a training trail with no warmup on CIFAR10 datasets (unconditional). We use a batch size of 256 and a learning rate of $1e-4$. We use the 2nd-order Heun numerical solver to sample. The result shows that VSDM has a smooth convergence performance.

\subsection{Time Series Forecasting}

We use multivariate probabilistic forecasting as a real-world \emph{conditional} modeling task. Let $\{ (t_1, \bx_1), \dots, (t_n, \bx_n) \}$, $\bx \in \mathbb{R}^d$, denote a single multivariate time series. Given a dataset of such time series we want to predict the next $P$ values $\bx_{n+1}, \dots, \bx_{n + P}$. In probabilistic modeling, we want to generate forecasts from learned $p(\bx_{n+1:n + P} | \bx_{1:n} )$.

The usual approach is to have an encoder that represents a sequence $\bx_{1:i}$ with a fixed-sized vector $\bm{h}_i \in \mathbb{R}^h$, $\forall i$, and then parameterize the output distribution $p(\bx_{i+1} | \bm{h}_i)$. At inference time we encode the history into $\bm{h}_n$ and sample the next value from $p(\bx_{n+1} | \bm{h}_n)$, then use $\bx_{n+1}$ to get the updated $\bm{h}_{n+1}$ and repeat until we obtain $\bx_{n + P}$.

In the previous works, the output distribution has been specified with a Copulas \cite{salinas2019high} and denoising diffusion \cite{rasul2021autoregressive}. We augment our approach to allow conditional generation which requires only changing the model to include the conditioning vector $\bm{h}_i$. For that we adopt the U-Net architecture. 
We use the LSTM neural network as a sequence encoder.

We use three real-world datasets, as described in Appendix~\ref{app:forecasting}. We compare to the SGM and the denoising diffusion approach from \citet{rasul2021autoregressive} which we refer to as DDPM. Table~\ref{tab:forecasting_results} shows that our method matches or outperforms the competitors. Figure \ref{fig:ts_demo} is a demo for conditional time series generation and more details are presented in Figure~\ref{fig:electricity} to demonstrate the quality of the forecasts.

\begin{table}[]
\begin{sc}
    \centering
    \caption[Table caption text]{\small{Forecasting results (lower is better).}}
    \begin{tabular}{c|ccc}
         \hline
         \scriptsize{CRPS-sum} & \scriptsize{Electricity} & \scriptsize{Exchange rate} & \scriptsize{Solar} \\
         \hline
         \scriptsize{DDPM}       & \textbf{\scriptsize{0.026$\pm$0.007}} & \scriptsize{0.012$\pm$0.001} & \scriptsize{0.506$\pm$0.058} \\
         \scriptsize{SGM}        & \scriptsize{0.045$\pm$0.005} & \scriptsize{0.012$\pm$0.002} & \scriptsize{0.413$\pm$0.045} \\
         \scriptsize{VSDM (our)} & \scriptsize{0.038$\pm$0.006} & \textbf{\scriptsize{0.008$\pm$0.002}} & \textbf{\scriptsize{0.395$\pm$0.011}} \\
         \hline
    \end{tabular}
    \label{tab:forecasting_results}
\end{sc}
\end{table}

\vspace{-0.1in}
\begin{figure}
\centering
\includegraphics[width=0.4\textwidth]{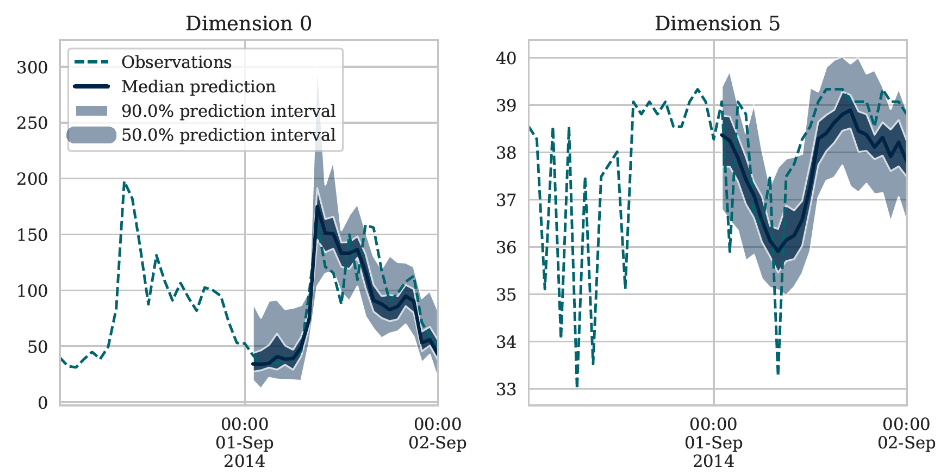}
\vspace{-2mm}
\caption{Example for Electricity for 2 (out of 370) dimensions.}
\label{fig:ts_demo}
\vspace{-2mm}
\end{figure}

\vskip -0.5in
\section{Conclusions and Future Works}

The Schr\"odinger bridge diffusion model offers a principled approach to solving optimal transport, but estimating the intractable forward score relies on implicit training through costly simulated trajectories. To address this scalability issue, we present the variational Schr\"odinger diffusion model (VSDM), utilizing linear variational forward scores for simulation-free training of backward score functions. Theoretical foundations leverage stochastic approximation theory, demonstrating the convergence of variational scores to local equilibrium and highlighting the variational gap in optimal transport. Empirically, VSDM showcases 
the strength of generating data with anisotropic shapes and yielding the desired straighter transport paths for reducing the number of functional evaluations. VSDM also exhibits scalability in handling large-scale image datasets without requiring warm-up initializations. In future research, we aim to explore the critically damped (momentum) acceleration \citep{SGM_damped} and Hessian approximations to develop the ``ADAM'' alternative of diffusion models.

\section*{Acknowledgements}

We thank Valentin De Bortoli, Tianyang Hu, and the anonymous reviewers for their valuable insights, and Anderson Schneider for his careful proofreading and help with typos.

\section*{Impact Statement}

This paper proposed a principled approach to accelerate the training and sampling of generative models using optimal transport. This work will contribute to developing text-to-image generation, artwork creation, and product design. However, it may also raise challenges in the fake-content generation and pose a threat to online privacy and security.

\bibliography{mybib, other}
\bibliographystyle{icml2024}

\newpage
\onecolumn
\appendix

\begin{Large}
\begin{center}
    \textbf{Supplementary Material for \textbf{``Variational Schr\"odinger Diffusion Models''}}
\end{center}
\end{Large}
$\newline$
In section \ref{diagonal_matrix}, we study the closed-form expression of matrix exponential for diagonal and time-invariant $\bD_t$; In section \ref{app:SA_process}, we study the convergence of the adaptive diffusion process; In section \ref{app:variational_gap}, we study the variational gap of the optimal transport and discuss its connections to Gaussian Schr\"odinger bridge; In section \ref{supp_exp_details}, we present more details on the empirical experiments.

\paragraph{Notations:} $\mathcal{X}$ is the state space for the data $\bx$; $\overleftarrow\bx^{(k)}_{nh}$ is the $n$-th backward sampling step with a learning rate $h$ at the $k$-th stage. $\eta_k$ is the step size to optimize $\bA$. $\mathcal{A}$ is the (latent) state space of $\bA$; $\bA_t^{(k)}$ is the forward linear score estimator at stage $k$ and time $t$, $\bA_t^{\star}$ is the equilibrium  of Eq.\eqref{mean_field_perturbed} at time $t$. $\nabla\overrightarrow{\mathcal{L}}_t$ is the random field in the stochastic approximation process and also the loss \eqref{VSDMf} at time $t$; $\nabla \overrightarrow{\bL}_t$ is the mean field with the equilibrium $\bA_t^{\star}$. Given a fixed $\bA_t^{(k)}$ at step $k$, $\nabla \log  \overrightarrow{\rho}_{t}^{(k)}$ (resp. $\nabla \log  \overrightarrow{\rho}_{t|0}^{(k)}$) is the (resp. conditional) forward score function of Eq.\eqref{mSGM-forward} at time $t$ and step $k$; $\bA_t^{(k)}$ yields the approximated score function $s_{t}^{(k)}$ and $\overleftarrow{\rho}_{t}^{(k)}$ is the distribution of the continuous-time interpolation of the discretized backward SDE \eqref{backward_DSM_discrete}.

\section{Closed-form Expression with Diagonal and Time-Invariant $\bD_t$}\label{diagonal_matrix}

In this section, we give the proof of Corollary \ref{corr:diagonal}. 

\begin{proof}
    Denote $\bD_t =\bm{\Lambda} \coloneqq \operatorname{diag}(\bm{\lambda})$, where $\lambda_i \geq 0,~ \forall 1 \leq i \leq d$, and $\sigma_t^2 \coloneqq \int_0^t \beta_s \mathrm{d}s$, then by Eq. \eqref{cov_dynamics}, we have 
    \begin{align*}
        \begin{pmatrix}
        \bC_t \\
        \bH_t
      \end{pmatrix} &=
      \exp\Bigg[
      \begin{pmatrix}
        -\frac{1}{2}[\beta\bD]_t & [\beta\bI]_t \\
        \bm{0} &  \frac{1}{2}[\beta\bD^{\intercal}]_t
      \end{pmatrix}
      \Bigg]
            \begin{pmatrix}
              {\bSigma}_0 \\
              {\bI} 
            \end{pmatrix} = \exp(\bM_t)\begin{pmatrix}
              {\bSigma}_0 \\
              {\bI}
            \end{pmatrix}.
    \end{align*}
Here $[\beta\bD]_t = \int_{0}^t \beta_s \bD_t \mathrm{d}s = \sigma_t^2 \bm{\Lambda}$. The matrix $\bM_t$ is defined as 
\begin{align*}
    \bM_t = \begin{pmatrix}
        -\frac{1}{2}\sigma_t^2 \bm{\Lambda} & \sigma_t^2\mathbf{I}\\
        \bm{0} &  \frac{1}{2}\sigma_t^2 \bm{\Lambda}
      \end{pmatrix}
\end{align*}
Therefore, we have 
\begin{align*}
&\bM_t^2 = \begin{pmatrix}
        (-\frac{1}{2}\sigma_t^2 \bm{\Lambda})^2 & \bm{0}\\
        \bm{0} &  (\frac{1}{2}\sigma_t^2 \bm{\Lambda})^2
      \end{pmatrix},~~ 
      \bM_t^3 = \begin{pmatrix}
        (-\frac{1}{2}\sigma_t^2 \bm{\Lambda})^3 & \sigma_t^2 (\frac{1}{2}\sigma_t^2 \bm{\Lambda})^2\\
        \bm{0} &  (\frac{1}{2}\sigma_t^2 \bm{\Lambda})^3
      \end{pmatrix},~~ \\
      &\bM_t^4 = \begin{pmatrix}
        (-\frac{1}{2}\sigma_t^2 \bm{\Lambda})^4 & \bm{0}\\
        \bm{0} &  (\frac{1}{2}\sigma_t^2 \bm{\Lambda})^4,~~
      \end{pmatrix},~~ 
      \bM_t^5 = \begin{pmatrix}
        (-\frac{1}{2}\sigma_t^2 \bm{\Lambda})^5 & \sigma_t^2 (\frac{1}{2}\sigma_t^2 \bm{\Lambda})^4\\
        \bm{0} &  (\frac{1}{2}\sigma_t^2 \bm{\Lambda})^5
      \end{pmatrix},~~ ...
\end{align*}
According to the definition of matrix exponential, we have
\begin{align*}
    \exp(\bM_t) & = [\mathbf{I} + \frac{1}{1!}\bM_t + \frac{1}{2!} \bM_t^2 + \frac{1}{3!}\bM_t^3 + ...] \\
    & = \begin{pmatrix}
        \exp(-\frac{1}{2}\sigma_t^2 \bm{\Lambda}) & \Big[ \sigma_t^2\mathbf{I} + \frac{1}{3!}\sigma_t^2(\frac{1}{2}\sigma_t^2 \bm{\Lambda})^2 + \frac{1}{5!}(\frac{1}{2}\sigma_t^2 \bm{\Lambda})^4 + ...\Big]\\
        \bm{0} &  \exp(\frac{1}{2}\sigma_t^2 \bm{\Lambda})
      \end{pmatrix}\\
      & = \begin{pmatrix}
        \exp(-\frac{1}{2}\sigma_t^2 \bm{\Lambda}) & \frac{\sigma_t^2}{\frac{1}{2}\sigma_t^2\bm{\Lambda}}\Big[ (\frac{1}{2}\sigma_t^2 \bm{\Lambda})^1 + \frac{1}{3!}\sigma_t^2(\frac{1}{2}\sigma_t^2 \bm{\Lambda})^3 + \frac{1}{5!}(\frac{1}{2}\sigma_t^2 \bm{\Lambda})^5 + ... \Big]\\
        \bm{0} &  \exp(\frac{1}{2}\sigma_t^2 \bm{\Lambda})
      \end{pmatrix}\\
      & = \begin{pmatrix}
        \exp(-\frac{1}{2}\sigma_t^2 \bm{\Lambda}) & \bm{\Lambda}^{-1}\Big[ \exp(\frac{1}{2}\sigma_t^2 \bm{\Lambda}) - \exp(-\frac{1}{2}\sigma_t^2 \bm{\Lambda}) \Big]\\
        \bm{0} &  \exp(\frac{1}{2}\sigma_t^2 \bm{\Lambda})
      \end{pmatrix}.\\
\end{align*}
Notice that, when we have ${\bSigma}_0=\bm{0}$, the expression can be simplified as follows
\begin{align*}
        \begin{pmatrix}
        \bC_t \\
        \bH_t
      \end{pmatrix} & = \exp(\bM_t)\begin{pmatrix}
              \bm{0} \\
              {\bI} 
            \end{pmatrix} = \begin{pmatrix}
              \bm{\Lambda}^{-1}\Big[ \exp(\frac{1}{2}\sigma_t^2 \bm{\Lambda}) - \exp(-\frac{1}{2}\sigma_t^2 \bm{\Lambda}) \Big] \\
              \exp(\frac{1}{2}\sigma_t^2 \bm{\Lambda})
            \end{pmatrix}.
\end{align*}
Therefore, $\bC_t\bH_t^{-1} = \bm{\Lambda}^{-1} \big\{ \mathbf{I} - \exp (-\sigma_t^2 \bm{\Lambda})\big\}$. As a result, the corresponding forward transition writes 
\begin{align*}
    \bm{\mu}_{t|0} &= \exp (-\frac{1}{2}\sigma_t^2 \bm{\Lambda})\bx_0\\
    \bL_t &= \sigma_t \bm{\Lambda}^{-\frac{1}{2}}\sqrt{\mathbf{I} - \exp(-\sigma_t^2 \bm{\Lambda})}.
\end{align*}
\end{proof}

\section{Stochastic Approximation}
\label{app:SA_process}

Stochastic approximation (SA), also known as the Robbins–Monro algorithm \citep{RobbinsM1951, Albert90} offers a conventional framework for the study of adaptive algorithms. The stochastic approximation algorithm works by repeating the sampling-optimization iterations in the dynamic setting in terms of simulated trajectories. We present our algorithm in Algorithm \ref{alg:SA_algorithm}.

\begin{algorithm*}
   \caption{The (dynamic) stochastic approximation (SA) algorithm. The (dynamic) SA is a theoretical formulation of Algorithm \ref{VSDM_alg}. We assume optimizing the loss function \eqref{MDM_loss} yields proper score estimations $s^{(k)}_{t}$ at each stage $k$ and time $t$ to approximate $\nabla \log  \overrightarrow{\rho}_{t}^{(k)}(\overrightarrow\bx_t | \overrightarrow\bx_0)$ in Eq.\eqref{multi_backward}.}
   \label{alg:SA_algorithm}
\begin{algorithmic}
\REPEAT
   \STATE{\textbf{Simulation}: Sample the backward process from \eqref{backward_DSM} given a fixed $\bA_{nh}^{(k)}$}
    \begin{equation}
     \overleftarrow\bx^{(k+1)}_{(n-1)h}= \left(-\frac{1}{2}\big(\bI-2\bA_{nh}^{(k)}\big)\beta_{nh} \overleftarrow\bx^{(k+1)}_{nh} - {\beta_{nh}} s^{(k+1)}_{nh}\big(\overleftarrow\bx^{(k+1)}_{nh}\big)\right) h+\sqrt{\beta_{nh} h} \bxi_n,\label{backward_DSM_discrete}
\end{equation}
where $\overleftarrow\bx^{(k+1)h}_{(N-1)}\sim \mathrm{N}(\bm{0}, \bSigma^{(k)}_{(N-1)h|0})$, $n\in \{1,2,\cdots, N-1\}$ and $h$ is the learning rate for the backward sampling \eqref{backward_DSM} via the Euler–Maruyama (EM) discretization. $\bxi_n$ denotes the standard Gaussian vector at the sampling iteration $n$. 
   \STATE{\textbf{Optimization}: } Minimize the implicit forward loss function \eqref{VSDMf}
   \begin{equation}\label{dsa_iterate}
       \bA_{nh}^{(k+1)}=\bA_{nh}^{(k)}-\eta_{k+1} \nabla\overrightarrow{\mathcal{L}}_{nh}(\bA_{nh}^{(k)}; \overleftarrow\bx^{(k+1)}_{nh}),
   \end{equation}
   where $\nabla\overrightarrow{\mathcal{L}}_{nh}(\bA_{nh}^{(k)}; \overleftarrow\bx^{(k+1)}_{nh})$ is the (dynamic) random field and $\eta_k$ is the step size. $n\in\{0, 1, \cdots, N-1\}$.
   \UNTIL{\text{Stage} $k=k_{\max}$}
    \vskip -1 in
\end{algorithmic}
\end{algorithm*}

To facilitate the analysis, we assume we only make a one-step sampling in Eq.\eqref{dsa_iterate}. Note that it is not required in practice and multiple-step extensions can be employed to exploit the cached data more efficiently. The theoretical extension is straightforward and omitted in the proof. We also slightly abuse the notation for convenience and generalize $\bA_{nh}$ to $\bA_{t}$.

Theoretically, the primary objective is to show  the iterates \eqref{dsa_iterate_main} follow
the trajectories of the dynamical system asymptotically:
\begin{equation}
\label{mean_field_perturbed_ode}
    \dd \bA_t = \nabla\overrightarrow{\bL}_t(\bA_t)\dd s,
\end{equation}
where $\nabla \overrightarrow{\bL}_t(\bA_t)$ is the mean field defined as follows:
\begin{equation}
\begin{split}
\label{mean_field_perturbed}
\nabla \overrightarrow{\bL}_t(\bA_t)&=\int_{\MX} \nabla\overrightarrow{\mathcal{L}}_t(\bA_t; \overleftarrow\bx^{(\cdot)}_t)  \overleftarrow{\rho}_{t}(\dd\overleftarrow\bx^{(\cdot)}_t).
\end{split}
\end{equation}

We denote by $\bA_t^{\star}$ the solution of $\nabla \overrightarrow{\bL}_t(\bA_t^{\star})=\bm{0}$. Since the samples simulated from $\overleftarrow{\rho}_{t}$ are slightly biased due to the convergence of forward process, discretization error, and score estimation errors as shown in Theorem \ref{theorem:quality_of_data}. We expect the mean field is also biased with a perturbed equilibrium. However, by the perturbation theory \citep{Eric}, the perturbation is mild and controlled by the errors in Theorem \ref{theorem:quality_of_data}. Hence although $\bA_t^{\star}$ is not the optimal linear solution in terms of optimal transport, it still yields efficient transportation plans.

Since the exclusive (reverse) KL divergence is known to approximate a single mode (denoted by $\bA_t^{\star}$) in fitting multi-modal distributions, we proceed to assume the following regularity conditions for the solution $\bA_t^{\star}$ and the neighborhood of $\bA_t^{\star}$.

\begin{assump}[Regularity]
\label{ass_local_state_space}

(Positive definiteness) For any $t\geq 0$ and $\bA_t \in \mathcal{A}$, there exists a constant $\lambda_{\min}>0$ s.t. $\lambda_{\min} \bI \preccurlyeq \bD_t=\bI - 2\bA_t$, where $\bA\preccurlyeq \bB$ means $\bB-\bA$ is semi positive definite. (Locally strong convexity) For any stable local minimum $\bA_t^{\star}$ with $\nabla \overrightarrow{\bL}_t(\bA_t^{\star})= \bm{0}$, there is always a neighborhood $\bTheta$ s.t. $\bA_t^{\star}\in \bTheta\subset \mathcal{A}$ and $\overrightarrow{\bL}_t$ is strongly convex in $\bTheta$, i.e. there exists fixed constants $M>m>0$ s.t. for $\forall \bA \in \bTheta$, $m\bI \preccurlyeq \frac{\partial^2 \overrightarrow{\bL}_t}{\partial \bA^2}(\bA)\preccurlyeq M \bI$. 
\end{assump}

The first part of the above assumption is standard and can be achieved by an appropriate regularization during the training; the second part only assumes the strong convexity for a small neighborhood $\bTheta$ of the optimum $\bA_t^{\star}$. As such, when conditions for Eq.\eqref{final_sa_convergence} hold, we can apply the induction method to make sure all the subsequent iterates of $\bA_t^{(k)}$ stay in the same region $\bTheta$ and converge to the local minimum $\bA_t^{\star}$. For future works, we aim to explore the connection between $m$ and $\lambda_{\min}$.

Next, we lay out three standard assumptions following \citet{chen2023improved} to conduct our analysis. Similar results are studied by \citet{lee2022convergence, Sitan_22_sampling_is_easy} with different score assumptions.

\begin{assump}[Lipschitz Score]
\label{ass_smoothness}
The score function $\nabla\log  \overrightarrow{\rho}_{t}$ ($\nabla\log  \overrightarrow{\rho}_{t, \bA}$)\footnote{We abstain from using $\nabla\log  \overrightarrow\rho_{t, \bA_t}$ for the sake of clarity. The smoothness w.r.t. $\bA_t$ is only used in Eq.\eqref{closedness}. When its use may lead to confusion elsewhere, we employ the $\nabla\log  \overrightarrow{\rho}_{t, \bA}$ notation.} is $L$-Lipschitz in both $\bx$ and $\bA$ for any $t\in[0, T]$. For any $\bA,\bB\in\mathcal{A}$ and any $\bx, \by \in \mathcal{X}$, we have
\begin{equation*}
\begin{split}
        \|\nabla\log  \overrightarrow{\rho}_{t, \bA}(\bx)-\nabla\log  \overrightarrow{\rho}_{t, \bA}(\by)\|_2 &\leq L \|\bx-\by\|_2 \\
        \|\nabla\log  \overrightarrow{\rho}_{t, \bA}(\bx)-\nabla\log  \overrightarrow{\rho}_{t, \bB}(\by)\|_2 &\leq L \|\bA-\bB\| \\
\end{split}
\end{equation*}
where $\|\cdot\|_2$ is the standard $L^2$ norm and $\|\cdot\|$ is matrix norm.

\end{assump}

\begin{assump}[Second Moment Bound]
\label{ass_moment_bound}
The data distribution has a bounded second moment $\mathfrak{m}_2^2:=\E_{\rho_{\text{data}}}[\|\cdot\|_2^2]<\infty$.
\end{assump}

\begin{assump}[Score Estimation Error]
\label{ass_score_estimation}

For all $t\in[0, T]$, and any $\bA_t$, we have some estimation error . 
$$\E_{\overrightarrow{\rho}_{t}}[\|s_{t}- \nabla \log \overrightarrow{\rho}_{t}\|_2^2]\leq \epsilon_{\text{score}}^2.$$

\end{assump}

We first use the multivariate diffusion to train our score estimators  $\{s^{(k)}_{t}\}_{n=0}^{N-1}$ via the loss function \eqref{MDM_loss} based on the pre-specified $\bA_t^{(k)}$. Following \citet{chen2023improved}, we can show the generated samples based on $\{s^{(k)}_{t}\}_{n=0}^{N-1}$ are close in distribution to the ideal samples in Theorem  \ref{theorem:quality_of_data}. The novelty lies in the extension of \textcolor{dark2blue}{single-variate diffusions} to \textcolor{dark2blue}{multi-variate diffusions}.

Next, we use the stochastic approximation theory to prove the convergence of $\bA_t^{(k)}$ to a local equilibrium $\bA_t^{\star}$ in Theorem \ref{theorem_L2_convergence}. 
In the end, we adapt Theorem  \ref{theorem:quality_of_data} again to show the adaptively generated samples are asymptotically close to the samples based on the optimal $\bA_t^{\star}$ in Theorem \ref{theorem_adaptive_sampling}, which further optimizes the transportation plans through a variational formulation. To facilitate the understanding, we summarize the details as follows
\begin{equation*}
\begin{split}
    &\small{\text{Sample via $\bA_t^{(k)}$}}\qquad\qquad\qquad\text{Random Field} 
    \qquad\qquad\   \text{Mean Field} \qquad\qquad\quad  \text{Convergence of $\bA_t^{(k)}$} \qquad\quad\quad \text{Sample via $\bA_t^{\star}$} \\
    &\footnotesize{\text{$s_{t}^{(k)}$} \xRightarrow[\text{Theorem  \ref{theorem:quality_of_data}}]{\text{Backward Sampling}}\nabla\overrightarrow{\mathcal{L}}_t(\bA_t^{(k)}; \overleftarrow\bx^{(k+1)}_t) 
    \xRightarrow[]{\text{Eq.\eqref{mean_field_perturbed}}}
    \nabla \overrightarrow{\bL}_t(\bA_t^{(k)})
    \xRightarrow[\text{Theorem \ref{theorem_L2_convergence}}]{\text{Convergence }} 
    \bA_t^{(k)}\rightarrow\bA_t^{\star} \xRightarrow[\text{Theorem \ref{theorem_adaptive_sampling}}]{\text{Adaptive Sampling}} \text{$\lim_{k\rightarrow\infty}\overleftarrow\bx_t^{(k+1)}$}}.
\end{split}
\end{equation*}

\paragraph{Proof of Sketch}
\begin{itemize}
    \item \textbf{Part \ref{part_1}}: The generated samples (backward trajectories) approximate the ideal samples from the fixed $\bA_t^{(k)}$.

    \item \textbf{Part \ref{part_2}}: We employ the SA theory to show the convergence $\bA_t^{(k)}$ to the optimal estimator $\bA_t^{\star}$.

    \item\textbf{Part \ref{part_4}}: The adaptively generated samples approximate the ideal samples from the optimal $\bA_t^{\star}$ asymptotically.
\end{itemize}

\subsection{Convergence of Approximated Samples with a Fixed $\bA_t$}
\label{part_1}

The following result is majorly adapted from Theorem 2.1 of \citet{chen2023improved}, where the \textcolor{dark2blue}{single-variate diffusions} are extended to the general \textcolor{dark2blue}{multi-variate diffusions}. 

Recall that the forward samples $\bx_t$ are sampled by~\eqref{mSGM-forward} given a fixed $\bA_t$, we denote the density of $\bx_t$ by $\overrightarrow{\rho}_{t}$ with $\overrightarrow{\rho}_{0} = \rho_{\text{data}}$. To facilitate the proof, we introduce an auxiliary variable $\by_t$ simulated from $\text{\eqref{mSGM-forward}}$ with $\by_0\sim \mathrm{N}(\bm{0}, \bI)$ such that $\by_t$ is always a Gaussian distribution at time $t$ and $\mathrm{KL}(\rho_{\text{data}}\|\mathrm{N}(\bm{0}, \bI))$ is well defined (not applicable to deterministic initializations for $\by_0$). We denote the auxiliary distribution of $\by_t$ at time $t$ by $\overrightarrow{\rho}_{t}^{\circ}$.  For a fixed $T>0$ and score estimations  $s_{t}$, let $\overleftarrow{\rho}_{t}$ be the distribution of the continuous-time interpolation of the discretized backward SDE from $t=T$ to $0$ with $\overleftarrow{\rho}_{T} = \overrightarrow{\rho}_{T}^{\circ}$. Then generation quality is measured by the distance between $\overleftarrow{\rho}_{0}$ and $\rho_{\text{data}}$.

\begin{theorem}[Generation quality]\label{theorem:quality_of_data}
    Assume assumptions \ref{ass_smoothness}, \ref{ass_moment_bound}, and \ref{ass_score_estimation} hold. Given a fixed $\bA_t$ by assumption \ref{ass_local_state_space}, the generated data distribution via the EM discretization of Eq.\eqref{backward_DSM} is close to the data distributions $\rho_{\text{data}}$ such that
    \begin{equation*}
        \mathrm{TV}(\overleftarrow{\rho}_{0}, \rho_{\text{data}})\lesssim \underbrace{\sqrt{\mathrm{KL}(\rho_{\text{data}}\|\gamma^d)} \exp(-T)}_{\text{convergence of forward process}} + \underbrace{(L\sqrt{dh} +  \mathfrak{m}_2 h)\sqrt{T}}_{\text{EM discretization}} + \underbrace{\epsilon_{\text{score}}\sqrt{T}}_{\text{score estimation}},
    \end{equation*}

    where $\gamma^d$ is the standard Gaussian distribution.
\end{theorem}
\begin{proof}
Following~\citet{chen2023improved}, we employ the chain rule for KL divergence and obtain:
\begin{align*}
\mathrm{KL}(\rho_{\text{data}}\|\overleftarrow{\rho}_{0}) \leq \mathrm{KL}(\overrightarrow{\rho}_{T} \| \overleftarrow{\rho}_{T}) + \E_{\overrightarrow{\rho}_{T}(\bx)}[\mathrm{KL}(\overrightarrow\rho_{0|T}(\cdot\|\bx)| \overleftarrow{\rho}_{0|T}(\cdot\|\bx)],
\end{align*}
where $\overrightarrow\rho_{0|T}$ is the conditional distribution of $\bx_0$ given $\bx_T$ and likewise for $\overleftarrow{\rho}_{0|T}$. Note that the two terms correspond to the convergence of the forward and reverse process respectively. We proceed to prove that 
\begin{align*}
    & \text{Part I: Forward process}\quad\quad \mathrm{KL}(\overrightarrow{\rho}_{T} \| \overleftarrow{\rho}_{T})=\mathrm{KL}(\overrightarrow{\rho}_{T} \| \overrightarrow{\rho}_{T}^{\circ}) \lesssim \mathrm{KL}(\rho_{\text{data}}\|\gamma^d) e^{-T},\\
     & \text{Part II: Backward process}\quad \E_{\overrightarrow{\rho}_{T}(\bx)}[\mathrm{KL}(\overrightarrow\rho_{0|T}(\cdot|\bx)\| \overleftarrow{\rho}_{0|T}(\cdot|\bx)] \lesssim (L^2dh + m_2^2 h^2)T + \epsilon_{\text{score}}^2 T.
\end{align*}

\text{Part I:} By the Fokker-Plank equation, we have
\begin{align*}
    \frac{\dd}{\dd t}\mathrm{KL}(\overrightarrow\rho_{t}\|\overrightarrow{\rho}_{t}^{\circ}) & = -\frac12 \beta_t J_{\overrightarrow{\rho}_{t}^{\circ}}(\overrightarrow\rho_{t})
\end{align*}
where
\begin{align*}
    J_{\overrightarrow{\rho}_{t}^{\circ}}(\overrightarrow\rho_{t}) = \int \overrightarrow\rho_{t}(x)\bigg\|\nabla\ln\frac{\overrightarrow\rho_{t}(\bx)}{\overrightarrow{\rho}_{t}^{\circ}(\bx)}\bigg\|^2 \dd \bx
\end{align*}
is the relative Fisher information of $\overrightarrow\rho_{t}$ with respect to $\overrightarrow{\rho}_{t}^{\circ}$. Note that for all $t\geq 0$, $\overrightarrow{\rho}_{t}^{\circ}$ is a Gaussian distribution and hence satisfies the log-Sobolev inequality ~\cite{vempala2022rapid}. It follows that
\begin{align*}
   \mathrm{KL}(\overrightarrow\rho_{t}\|\overrightarrow{\rho}_{t}^{\circ})\leq \frac1{2\alpha_t} J_{\overrightarrow{\rho}_{t}^{\circ}}(\overrightarrow\rho_{t}),
\end{align*}
where $\alpha_t$ is the log-Sobolev constant of $\overrightarrow{\rho}_{t}^{\circ}$. This implies that
\begin{align}
   \frac{\dd}{\dd t}\mathrm{KL}(\overrightarrow\rho_{t}\|\overrightarrow{\rho}_{t}^{\circ}) \leq - \alpha_t \beta_t \mathrm{KL}(\overrightarrow\rho_{t}\|\overrightarrow{\rho}_{t}^{\circ}).\notag
\end{align}
Applying the Gr\"onwall's inequality yields
\begin{align*}
    \mathrm{KL}(\overrightarrow\rho_{t}\|\overrightarrow{\rho}_{t}^{\circ}) \leq e^{-\int_0^t \alpha_s \beta_s \dd s}\mathrm{KL}(\overrightarrow\rho_{0}\|\overrightarrow{\rho}_{0}^{\circ}) \leq e^{- \alpha \int_0^t \beta_s \dd s}\mathrm{KL}(\overrightarrow\rho_{0}\|\overrightarrow{\rho}_{0}^{\circ}),
\end{align*}
where the last inequality is followed by Lemma \ref{low_LSI} and $\alpha$ is a lower bound estimate of the LSI constant $\inf_{t\in [0, T]}\alpha_t$. 

Then by Pinsker's Inequality, we have
\begin{align*}
    \text{TV}(\overrightarrow\rho_{t}, \overrightarrow{\rho}_{t}^{\circ}) \leq \sqrt{2 \mathrm{KL}(\overrightarrow\rho_{t}\|\overrightarrow{\rho}_{t}^{\circ})} \leq \sqrt{2 e^{-\alpha \int_0^t \beta_s \dd s}\mathrm{KL}(\overrightarrow\rho_{0}\|\overrightarrow{\rho}_{0}^{\circ})} \lesssim \sqrt{\mathrm{KL}(\rho_{\text{data}}\|\gamma^d)} \exp(-t).
\end{align*}

\text{Part II:} The proof for the convergence of the reverse process is essentially identical to Theorem 2.1 of~\citet{chen2023improved}, with the only potential replacements being instances of $\|\bx_t - \bx_{kh}\|_2$ with $\|\bD_{T-t}(\bx_t - \bx_{kh})\|_2$. However, they are equivalent due to Assumption~\ref{ass_local_state_space}. Therefore, we omit the proof here. 

In conclusion, the convergence follows that
\begin{align*}
 \mathrm{KL}(\rho_{\text{data}}\|\overleftarrow{\rho}_{0}) \lesssim \mathrm{KL}(\rho_{\text{data}}\|\gamma^d) e^{-T} + (L^2dh + m_2^2h^2)T + \epsilon_{\text{score}}T.
\end{align*}
And we obtain the final result using the Pinsker's Inequality.
\qed
\end{proof}

\begin{lemma}[Lower bound of the log-Sobolev constant]\label{low_LSI} Under the same assumptions and setups in Theorem \ref{theorem:quality_of_data}, we have
    $$\inf_{t\in[0,T]} \alpha_t \geq \min\{1, \lambda_{\min}\} =: \alpha \sim O(1).$$
\end{lemma}

\begin{proof}
Consider the auxiliary process for $\by_t$: 
\begin{itemize}
    \item Randomness from the initial: By the mean diffusion in Eq.\eqref{mu_diffusion}, the conditional mean diffusion of $\by_t$ at time $t$, denoted by $\mu_{t,\by}$, follows that $\mu_{t, \by} = \mathbb{D}_t \mu_{0, \by}$, where $\mathbb{D}_t = e^{-\frac{1}{2} [\beta\bD]_t}$. Since $\by_0\sim \mathrm{N}(\bm{0}, \bI)$, we know $\mu_{t, \by}\sim\mathrm{N}(\bm{0}, \mathbb{D}_t \mathbb{D}_t^\intercal)$.
    \item Randomness from Brownian motion: the covariance diffusion induced by Brownian motion follows from $\bSigma_{t|0}$ in Eq.\eqref{sigma_diffusion}.
\end{itemize}
Since $\by_0\sim \mathrm{N}(\bm{0}, \bI)$ and $\by_t$ is an OU process in Eq.\eqref{mSGM-forward}, we know that $\by_t$ is always a Gaussian distribution at time $t\geq 0$ with mean $\bm{0}$. As such, we know that
\begin{equation}
    \label{auxiliary}
    \overrightarrow{\rho}_{t}^{\circ} = \mathrm{N}(\bm{0}, \mathbb{D}_t \mathbb{D}_t^\intercal + \bSigma_{t|0}).
\end{equation}

It follows that
\begin{align*}
    \text{TV}(\overrightarrow\rho_{t}, \overrightarrow{\rho}_{t}^{\circ}) \leq \sqrt{2 e^{-\int_0^t \alpha_s \beta_s \dd s}\mathrm{KL}(\overrightarrow\rho_{0}\|\overrightarrow{\rho}_{0}^{\circ})}.
\end{align*}
Now we need to bound the log-Sobolev constant $\alpha_t$ of $\overrightarrow{\rho}_{t}^{\circ}$. Let $\bSigma_t = \mathbb{D}_t\mathbb{D}_t^\intercal + \bSigma_{t|0}$. Recall that if a distribution $p$ is $\alpha$-strongly log-concave, then it satisfies the log-Sobolev inequality (LSI) with LSI constant $\alpha$~\cite{vempala2022rapid}. So for the Gaussian distribution $\overrightarrow{\rho}_{t}^{\circ}$, it suffices to bound the (inverse of) smallest eigenvalue of $\bSigma_t$. Recall from Eq.\eqref{sigma_diffusion} that $\bSigma_t$ satisfies the ODE
\begin{align*}
    \frac{\dd\bSigma_t}{\dd t} = -\frac12\beta_t(\bD_t \bSigma_t + \bSigma_t \bD_t^\intercal) + \beta_t \bI,\quad \bSigma_0 = \bI.
\end{align*}
Fix a normalized vector $\bx\in\mathbb R^d$ and denote $u_t = \bx^\intercal \bSigma_t \bx$ for $t\in[0, T]$. By the cyclical property of the trace, we have
\begin{align*}
    \bx^\intercal \bD_t \bSigma_t \bx = \text{Tr}(\bx^\intercal \bD_t \bSigma_t \bx) = \text{Tr}(\bD_t \bSigma_t \bx \bx^\intercal )\geq \lambda_{\min} \text{Tr}(\bSigma_t \bx \bx^\intercal ) = \lambda_{\min} u_t.
\end{align*}

It follows that
\begin{align*}
    \frac{\dd u_t}{\dd t} \leq -\lambda_{\min}\beta_t u_t + \beta_t. 
\end{align*}
Applying the Gr\"onwall's inequality tells us that
\begin{align*}
    u_t \leq \frac{1}{\lambda_{\min}} (1-e^{-\lambda_{\min} \int_0^T \beta_s \dd s}) + e^{-\lambda_{\min} \int_0^T \beta_s \dd s} \leq \max\{1, 1/\lambda_{\min}\}.
\end{align*}
Since $\bx$ can be any normalized vector, we have that the largest eigenvalue of $\bSigma_t$ is bounded by $\max\{1, 1/\lambda_{\min}\}$ and hence
\begin{align*}
    \inf_{t\in[0,T]} \alpha_t \geq \min\{1, \lambda_{\min}\} =: \alpha \sim O(1),
\end{align*}
where $\alpha_t$ is the log-Sobolev constant of $\overrightarrow{\rho}_{t}^{\circ}$.  \qed
\end{proof}

\paragraph{Remark:} \label{prior_remark}In our theoretical analysis, we introduced an auxiliary variable $\by_0\sim \gamma^d$ to make sure $\mathrm{KL}(\rho_{\text{data}}\| \gamma^d)$ is well defined. Moreover, the distribution of $\by_T$ is set to $\overrightarrow{\rho}_{T}^{\circ}$ in Eq.\eqref{auxiliary}. However, we emphasize that the introduction of $\by_t$ is only for theoretical analysis and we adopt a simpler prior $\mathrm{N}(\bm{0}, \bSigma_{T|0})$ instead of $\mathrm{N}(\bm{0}, \mathbb{D}_T \mathbb{D}_T^\intercal + \bSigma_{T|0})$ in Eq.\eqref{auxiliary} for convenience.

\subsection{Part II: Stochastic Approximation Convergence}
\label{part_2}

$\bA_t^{(k)}$ converges to $\bA_t^{\star}$ by tracking a mean-field ODE with some fluctuations along the trajectory. Before we prove the convergence, we need to show the stability property of the mean-field ODE such that small fluctuations of earlier iterates do not affect the convergence to the equilibrium. To that end, we construct a Lyapunov function $\mathbb{V}_t(\bA)=\frac{1}{2} m \|\bA-\bA_t^{\star}\|_2^2$ to analyze the local stability condition of the solution. This result shows that when the solution is close to the equilibrium $\bA_t^{\star}\in\bTheta\subset \bA$, $\bA_t$ will asymptotically track the trajectory of the mean field \eqref{mean_field_perturbed_ode} within $\bTheta$ when the step size $\eta_k\rightarrow 0$.

\begin{lemma}[Local stabiltity]\label{lemma_local_stability}
    Assume the assumptions \ref{ass_local_state_space} and \ref{ass_smoothness} hold. For any $\bA\in \bTheta$, the solution satisfies a local stability condition such that
        \begin{equation*}
        \label{local_stability}
            \langle \bA-\bA_t^{\star}, \nabla \mathbb{V}_t(\bA) \rangle =\langle \bA-\bA_t^{\star}, \nabla \overrightarrow{\bL}_t(\bA) \rangle \geq m \|\bA-\bA_t^{\star}\|_2^2. 
        \end{equation*}
\end{lemma}

\begin{proof}
By the smoothness assumption \ref{ass_smoothness} and Taylor expansion, for any $\bA\in\bTheta$, we have
\begin{equation}\label{taylor_expansion}
    \nabla \overrightarrow{\bL}_t(\bA)=\nabla  \overrightarrow{\bL}_t(\bA_t^{\star})+\textbf{Hess}\big[\overrightarrow{\bL}_t\big(\widetilde \bA\big)\big](\bA- \bA_t^{\star})=\textbf{Hess}\big[\overrightarrow{\bL}_t\big(\widetilde \bA\big)\big](\bA- \bA_t^{\star}),
\end{equation}
where $\textbf{Hess}\big[\overrightarrow{\bL}_t\big(\bA\big)\big]$ denotes the Hessian of $\overrightarrow{\bL}_t$ with $\bA$ at time $t$; $\widetilde \bA$ is some value between $\bA$ and $\bA_t^{\star}$ by the mean-value theorem. Next, we can get
\begin{equation*}
    \langle \bA-\bA_t^{\star}, \nabla \overrightarrow{\bL}_t(\bA) \rangle = \textbf{Hess}\big[\overrightarrow{\bL}_t\big(\widetilde \bA\big)\big]\|\bA- \bA_t^{\star}\|_2^2 \geq  m \|\bA-\bA_t^{\star}\|_2^2, 
\end{equation*}
where the last inequality follows by assumption \ref{ass_local_state_space}. \qed
\end{proof}

Additionally, we show the random field satisfies a linear growth condition to avoid blow up in tails.
          
\begin{lemma}[Linear growth]\label{lemma_linear_growth} Assume the assumptions \ref{ass_smoothness} and \ref{ass_moment_bound} hold. There exists a constant $C>0$ such that $\forall \bA_t^{(k)}\in\bTheta$ at the SA step $k$ and time $t$, the random field is upper bounded in $L^2$ such that
\begin{equation*}
    \E[\|\nabla\overrightarrow{\mathcal{L}_t}(\bA_t^{(k)}, \overleftarrow\bx^{(k+1)}_t)\|^2_2|\mathcal{F}_k] \leq C(1+\|\bA^{(k)}_t-\bA^{\star}_t\|_2^2):=C(1+\|\bG^{(k)}_t\|_2^2),
\end{equation*}
where the trajectory $ \overleftarrow\bx^{(k+1)}_t$ is simulated by \eqref{backward_DSM_discrete}; $\mathcal{F}_k$ is a $\sigma$-filtration formed by $(\overleftarrow\bx_t^{(1)}, \bA_t^{(1)}, \overleftarrow\bx_t^{(2)}, \bA_t^{(2)}, \cdots, \overleftarrow\bx_t^{(k)}, \bA_t^{(k)})$.

\end{lemma}

\begin{proof} By the unbiasedness of the random field, we have
\begin{equation}
\label{unbiasedness_}
\E[\nabla\overrightarrow{\mathcal{L}}_t(\bA_t^{(k)}; \overleftarrow\bx^{(k+1)}_t)-\nabla \overrightarrow{\bL}_t(\bA_t^{(k)})|\mathcal{F}_k]=\bm{0}.
\end{equation}

It follows that
    \begin{equation}
    \begin{split}
    \E[\|\nabla\overrightarrow{\mathcal{L}}_t(\bA_t^{(k)}; \overleftarrow\bx^{(k+1)}_t)\|^2_2|\mathcal{F}_k] &=\E[\|\nabla\overrightarrow{\mathcal{L}}_t(\bA_t^{(k)}; \overleftarrow\bx^{(k+1)}_t)-\nabla  \overrightarrow{\bL}_t(\bA_t^{(k)})+\nabla  \overrightarrow{\bL}_t(\bA_t^{(k)}))\|^2_2|\mathcal{F}_k]\\
    &=\E[\|\nabla\overrightarrow{\mathcal{L}}_t(\bA_t^{(k)}; \overleftarrow\bx^{(k+1)}_t)-\nabla  \overrightarrow{\bL}_t(\bA_t^{(k)})\|^2_2|\mathcal{F}_k]+\|\nabla  \overrightarrow{\bL}_t(\bA_t^{(k)})\|^2_2\\
        &\leq \sup \E[\|\nabla\overrightarrow{\mathcal{L}}_t(\bA_t^{(k)}; \overleftarrow\bx^{(k+1)}_t)-\nabla  \overrightarrow{\bL}_t(\bA_t^{(k)})\|^2_2|\mathcal{F}_k]+M^2  \|\bA_t^{(k)}-\bA_t^{\star}\|_2^2,
    \end{split}
    \end{equation}
where the last inequality follows by assumption \ref{ass_local_state_space} and Eq.\eqref{taylor_expansion}.

By assumption \ref{ass_smoothness} and \ref{ass_moment_bound} and the process \eqref{backward_DSM}, we know $\sup \E[\|\nabla\overrightarrow{\mathcal{L}}_t(\bA_t^{(k)}; \overleftarrow\bx^{(k+1)}_t)-\nabla  \overrightarrow{\bL}_t(\bA_t^{(k)})\|^2_2|\mathcal{F}_k]<\infty$. Denote by $C:= \max\{\sup \E[\|\nabla\overrightarrow{\mathcal{L}}_t(\bA_t^{(k)}; \overleftarrow\bx^{(k+1)}_t)-\nabla  \overrightarrow{\bL}_t(\bA_t^{(k)})\|^2_2|\mathcal{F}_k], M^2\} $, we can conclude that
    \begin{equation*}
        \E[\|\nabla\overrightarrow{\mathcal{L}}_t(\bA_t^{(k)}; \overleftarrow\bx^{(k+1)}_t)\|^2_2|\mathcal{F}_k] \leq C(1+\|\bA_t^{(k)}-\bA_t^{\star}\|_2^2).
    \end{equation*}

    \qed
\end{proof}

Next, we make standard assumptions on the step size following \citet{Albert90} (page 245).
\begin{assump}[Step size]
\label{ass_step_size}

The step size $\{\eta_{k}\}_{k\in \mathrm{N}}$ is a positive and decreasing sequence 
\begin{equation*} \label{a1}
\eta_{k}\rightarrow 0, \ \ \sum_{k=1}^{\infty} \eta_{k}=+\infty,\ \  \lim_{k\rightarrow \infty} \inf \left(2m  \dfrac{\eta_{k}}{\eta_{k+1}}+\dfrac{\eta_{k+1}-\eta_{k}}{\eta^2_{k+1}}\right):=\kappa>0.
\end{equation*}
A standard choice is to set $\eta_{k}:=\frac{A}{k^{\alpha}+B}$ for some $\alpha \in (\frac{1}{2}, 1]$ and some suitable constants 
 $A>0$ and $B>0$.
 \end{assump}

\begin{theorem}[Convergence in $L^2$]\label{theorem_L2_convergence}
    Assume assumptions \ref{ass_local_state_space}, \ref{ass_smoothness}, \ref{ass_moment_bound}, \ref{ass_score_estimation}, and \ref{ass_step_size} hold. The variational score $\bA_t^{(k)}$ in algorithm \ref{alg:SA_algorithm} converges to a local minimizer $\bA_t^{\star}$. In other words, given a large enough $k\geq k_0$, where $\eta_{k_0}\leq \frac{1}{2}$, we have
    \begin{equation*}
    \E[\|\bA_t^{(k)}-\bA_t^{\star}\|_2^2]\leq 2 \eta_{k},
\end{equation*}
where the expectation is taken w.r.t samples from $\overleftarrow{\rho}_{t}^{(k)}$.
\end{theorem}

\begin{proof} To show $\bA_t^{(k)}$ converges to $\bA_t^{\star}$, we first denote $\bG_t^{(k)}=\bA_t^{(k)}-\bA_t^{\star}$. Subtracting $\bA^{\star}$ on both sides of Eq.\eqref{dsa_iterate}:
\begin{align*}
    \bG_t^{(k+1)}&=\bG_t^{(k)}-\eta_{k+1} \nabla\overrightarrow{\mathcal{L}_t}(\bA_t^{(k)}; \overleftarrow\bx^{(k+1)}_t).
\end{align*}

By the unbiasedness of the random field, we have
\begin{equation}
\label{unbiasedness}
\E[\nabla\overrightarrow{\mathcal{L}_t}(\bA_t^{(k)}; \overleftarrow\bx_t^{(k+1)})-\nabla \overrightarrow{\bL}_t(\bA_t^{(k)})|\mathcal{F}_k]=\bm{0}.
\end{equation}

Taking the expectation in $L^2$, we have
\begin{equation*}
\begin{split}
    \label{main_iterate}
    \E[\|\bG_t^{(k+1)}\|_2^2|\mathcal{F}_k]&=\|\bG_t^{(k+1)}\|^2_2-2\eta_{k+1} \E\big[\langle \bG_t^{(k)}, \nabla\overrightarrow{\mathcal{L}_t}(\bA^{(k)}; \overleftarrow\bx_t^{(k+1)})\rangle\big] + \eta_{k+1}^2 \E\big[\|\nabla\overrightarrow{\mathcal{L}_t}(\bA_t^{(k)}; \overleftarrow\bx^{(k+1)}_t)\|^2_2|\mathcal{F}_k\big] \\
    &=\|\bG_t^{(k+1)}\|^2_2-2\eta_{k+1} \langle \bG_t^{(k)}, \nabla \overrightarrow{\bL}_t(\bA_t^{(k)})\rangle + \eta_{k+1}^2 \E\big[\|\nabla\overrightarrow{\mathcal{L}_t}(\bA_t^{(k)}; \overleftarrow\bx^{(k+1)}_t)\|^2_2|\mathcal{F}_k\big],
\end{split}
\end{equation*}
where the second equality is followed by the unbiasedness property in Eq.\eqref{unbiasedness}.

Applying the stepsize assumption \ref{ass_step_size}, we have 
\begin{equation*}
    \eta_{k+1} - \eta_k + 2m \eta_k\eta_{k+1} \geq C \eta_{k+1}^2. 
\end{equation*}

Then for $\eta_k\leq \frac{1}{2}$, we have
\begin{equation*}
    2(\eta_{k+1} - \eta_k + \eta_k\eta_{k+1}(2m -\eta_{k+1} C)) \geq C \eta_{k+1}^2. 
\end{equation*}

Rewrite the above equation as follows
\begin{equation*}\label{convergence_seq}
    2\eta_{k+1} \geq (1 - 2\eta_{k+1}m +C\eta_{k+1}^2)(2\eta_k) + C\eta_{k+1}^2. 
\end{equation*}


By the induction method, we have
\begin{itemize}
    \item Given some large enough $k\geq k_0$, where $\eta_{k_0}\leq \frac{1}{2}$, $\bA_t^{(k)}$ is in some subset $\bTheta$ \footnote{By assumption \ref{ass_local_state_space}, such $\bTheta\subset \mathcal{A}$ exists, otherwise it implies that the mean field function is a constant and conclusion holds as well.} of $\mathcal{A}$ that follows
    \begin{equation}
        \E[\|\bG_t^{k}\|_2^2]\leq 2\eta_k. \label{final_sa_convergence}
    \end{equation}
    \item Applying Eq.\eqref{main_iterate} and Eq.\eqref{convergence_seq}, respectively, we have
    \begin{align}
        \E[\|\bG_t^{(k+1)}\|_2^2|\mathcal{F}_k]&\leq (1-2\eta_{k+1}m)\E[\|\bG_t^{(k)}\|_2^2] + \eta_{k+1}^2 \E\big[\|\nabla\overrightarrow{\mathcal{L}_t}(\bA_t^{(k)}; \overleftarrow\bx^{(k+1)}_t)\|^2_2|\mathcal{F}_k\big]\notag\\
    &\leq (1-2\eta_{k+1}m + C\eta_{k+1}^2)\E[\|\bG_t^{(k)}\|_2^2] + C\eta_{k+1}^2,\notag\\
    &\leq (1-2\eta_{k+1}m + C\eta_{k+1}^2)(2\eta_k) + C\eta_{k+1}^2\label{final_sa_convergence2}\\
        &\leq 2\eta_{k+1}\notag,
    \end{align}
    where the first inequality is held by the stability property in Lemma \ref{lemma_local_stability} and the last inequality is followed by the growth property in Lemma \ref{lemma_linear_growth}.
\end{itemize}

Since $\bA_t^{\star}, \bA_t^{(k)} \in \bTheta$, Eq.\eqref{final_sa_convergence2} implies that $\bA_t^{(k+1)}\in \bTheta$, which concludes the proof. \qed
\end{proof}

\subsection{Part III: Convergence of Adaptive Samples based on The Optimal $\bA^{\star}$}
\label{part_4}

We have evaluated the sample quality in Theorem \ref{theorem:quality_of_data} based on a fixed $\bA_t$, which, however, may not be efficient in terms of transportation plans. To evaluate the sample quality in terms of the limiting optimal $\bA^{\star}$, we provide the result as follows:   
\begin{theorem}
\label{theorem_adaptive_sampling}
Given assumptions \ref{ass_local_state_space}-\ref{ass_step_size}, the generated sample distribution at stage $k$ is $\epsilon$-close to the exact sample distribution $\overrightarrow{\rho}^{\star}_{T}$ based on the equilibrium $\bA_t^{\star}$ such that
\begin{equation*}
        \mathrm{TV}(\overleftarrow{\rho}^{\star}_{0}, \rho_{\text{data}})\lesssim \sqrt{\mathrm{KL}(\rho_{\text{data}}\|\gamma^d)} \exp(-T) + (L\sqrt{dh} + L \mathfrak{m}_2 h)\sqrt{T} + (\epsilon_{\text{score}} + \sqrt{\eta_k})\sqrt{T}.
    \end{equation*}
\end{theorem}

\begin{proof}
By assumption \ref{ass_score_estimation}, for any $\bA^{(k)}_t\in \mathcal{A}$, we have
$$\E_{\overrightarrow{\rho}_{t}^{(k)}}[\|s_{t}^{(k)}- \nabla \log \overrightarrow{\rho}_{t}^{(k)}\|_2^2]\leq \epsilon_{\text{score}}^2.$$

Combining Theorem \ref{theorem_L2_convergence} and the smoothness assumption \ref{ass_smoothness} of the score function $\nabla\log  \overrightarrow{\rho}^{(k)}_{t}$ w.r.t $\bA^{(k)}_t$, we have
\begin{equation}\label{closedness}
    \E_{\overrightarrow{\rho}_{t}^{(k)}}[\|\nabla \log \overrightarrow{\rho}^{(k)}_{t}- \nabla \log \overrightarrow{\rho}^{\star}_{t}\|_2^2]\lesssim \eta_k.
\end{equation}

It follows that the score function $s_{t}^{(k)}$ is also close to the optimal $\nabla \log \overrightarrow{\rho}_{t}^{\star}$ in the sense that
\begin{equation}
\begin{split}\label{adaptive_score_error}
    &\quad\ \E_{\overrightarrow{\rho}_{t}^{(k)}}[\|s_{t}^{(k)}- \nabla \log \overrightarrow{\rho}_{t}^{\star}\|_2^2]\\
    &\lesssim \E_{\overrightarrow{\rho}_{t}^{(k)}}[\underbrace{\|s_{t}^{(k)}- \nabla \log \overrightarrow{\rho}_{t}\|_2^2}_{\text{by Assumption \ref{ass_score_estimation}}}] + \E_{\overrightarrow{\rho}^{(k)}_{t}}[\underbrace{\|\nabla \log \overrightarrow{\rho}_{t}^{(k)}- \nabla \log \overrightarrow{\rho}_{t}^{\star}\|_2^2}_{\text{by Eq.\eqref{closedness}}}] \\
    &\lesssim \epsilon_{\text{score}}^2+ \eta_k.
\end{split}
\end{equation}

Applying Theorem \ref{theorem:quality_of_data} with the adaptive score error in Eq.\eqref{adaptive_score_error} to replace $\epsilon_{\text{score}}^2$ concludes the proof.
\qed
\end{proof}

\paragraph{Remark:} The convergence of samples based on the adaptive algorithms is slightly weaker than the standard one due to the adaptive update, but this is necessary because $\bA_t^{\star}$ is more transport efficient than a vanilla $\bA_t$.

\section{Variational Gap}
\label{app:variational_gap}

Recall that the optimal forward SDE in the forward-backward SDEs \eqref{FB-SDE} follows that
\begin{equation}\label{forward_fb_sde}
    {\dd  \overrightarrow\bx_t}={\left[\bbf_t(\overrightarrow\bx_t) +  \beta_t\nabla\log\overrightarrow\psi_t(\overrightarrow\bx_t)\right]\dd t}+ \sqrt{\beta_t} \dd  \overrightarrow\bw_t.  
\end{equation}

The optimal variational forward SDE follows that
\begin{equation}\label{optimal_forward_linear_fb_sde}
    \dd \overrightarrow\bx_t=\left[\bbf_t(\overrightarrow\bx_t) +  {\beta_t} \bA_t^{\star}\overrightarrow\bx_t\right]\dd t+\sqrt{\beta_t} \dd \overrightarrow\bw_t.
\end{equation}

The variational forward SDE at the $k$-th step follows that
\begin{equation}\label{forward_linear_fb_sde}
    \dd \overrightarrow\bx_t=\left[\bbf_t(\overrightarrow\bx_t) +  {\beta_t} \bA_t^{(k)}\overrightarrow\bx_t\right]\dd t+\sqrt{\beta_t} \dd \overrightarrow\bw_t.
\end{equation}

Since we only employ a linear approximation of the forward score function, our transport is only sub-optimal. To assess the extent of this discrepancy, we leverage the Girsanov theorem to study the variational gap.  

We denote the law of the processes by $\mathrm{L}(\cdot)$ in Eq.\eqref{forward_fb_sde}, ${\mathrm{L}}^{\star}(\cdot)$ in Eq.\eqref{optimal_forward_linear_fb_sde} and ${\mathrm{L}}^{(k)}(\cdot)$ in Eq.\eqref{forward_linear_fb_sde}, respectively.

\begin{theorem}
Assume assumptions \ref{ass_smoothness} and \ref{ass_moment_bound} hold. Assume $\bbf_t$ and $\nabla\log\overrightarrow\psi_t$ are Lipschitz smooth and satisfy the linear growth condition. Assume the Novikov’s condition holds for $\forall \bA_t\in \mathcal{A}$, where $t\in[0, T]$:
\begin{equation*}
    \E\bigg[\exp\bigg(\frac{1}{2}\int_0^T \| {\beta_t} \bA_t\overrightarrow\bx_t - \beta_t\nabla\log\overrightarrow\psi_t(\overrightarrow\bx_t) \|_2^2 \dd t\bigg)\bigg]<\infty.
\end{equation*}

The variational gap (VG) via the linear parametrization is upper bounded by
\begin{equation*}
\begin{split}
    &{\mathrm{KL}({\mathrm{L}}\|{\mathrm{L}}^{\star})=\frac{1}{2} \int_0^T \E_{\overrightarrow{\rho}_t}\bigg[\beta_t \|\bA_t^{\star}\overrightarrow\bx_t - \nabla\log\overrightarrow\psi_t(\overrightarrow\bx_t)\|_2^2\dd t\bigg]}\\
    &\mathrm{KL}({\mathrm{L}}\|{\mathrm{L}}^{(k)}) \lesssim \eta_k + \mathrm{KL}({\mathrm{L}}\|{\mathrm{L}}^{\star}).
\end{split}
\end{equation*}

\end{theorem}

\begin{proof}

By Girsanov's formula \citep{Liptser01}, the Radon–Nikodym derivative of ${\mathrm{L}}(\cdot)$ w.r.t. ${\mathrm{L}}^{\star}(\cdot)$ follows that
\begin{equation*}
    \frac{\dd \mathrm{L}}{\dd \mathrm{L}^{\star}}\big(\overrightarrow\bx\big)=\exp\bigg(\int_0^T \sqrt{\beta_t} \bigg(\bA_t^{\star}\overrightarrow\bx_t - \nabla\log\overrightarrow\psi_t(\overrightarrow\bx_t) \bigg) \dd \bw_t - \frac{1}{2}\int_0^T \beta_t \|\bA_t^{\star}\overrightarrow\bx_t - \nabla\log\overrightarrow\psi_t(\overrightarrow\bx_t)\|_2^2\dd t\bigg),
\end{equation*}
where $\bw_t$ is the Brownian motion under the Wiener measure. Consider a change of measure \citep{oksendal2003stochastic, log_concave_sampling}
\begin{equation*}
    \bw_t=\widetilde\bw_t-\dd\big[\bw, \bM\big]_t, \quad \dd \bM_t = \big\langle \sqrt{\beta_t} \big(\bA_t^{\star}\overrightarrow\bx_t - \nabla\log\overrightarrow\psi_t(\overrightarrow\bx_t)\big), \dd\bw_t \big\rangle,
\end{equation*}
where $\widetilde \bw_t$ is a $\mathrm{L}$-standard Brownian motion and satisfies martingale property under the $\mathrm{L}$ measure.

Now the variational gap is upper bounded by
\begin{equation*}
\begin{split}
    \text{KL}({\mathrm{L}}(\cdot)\|\mathrm{L}^{\star}(\cdot)) &=-\E_{\mathrm{L}(\cdot)}\bigg[\log\frac{\dd \mathrm{L}(\cdot)}{\dd \mathrm{L}^{\star}(\cdot)}\bigg]\\
    &=\E_{\mathrm{L}(\cdot)}\bigg[\int_0^T \sqrt{\beta_t} \bigg(\bA_t^{\star}\overrightarrow\bx_t - \nabla\log\overrightarrow\psi_t(\overrightarrow\bx_t) \bigg) \dd \widetilde\bw_t + \frac{1}{2}\int_0^T \beta_t \|\bA_t^{\star}\overrightarrow\bx_t - \nabla\log\overrightarrow\psi_t(\overrightarrow\bx_t)\|_2^2\dd t\bigg]\\
    &=\frac{1}{2}\E_{\mathrm{L}(\cdot)}\bigg[\int_0^T \beta_t \|\bA_t^{\star}\overrightarrow\bx_t - \nabla\log\overrightarrow\psi_t(\overrightarrow\bx_t)\|_2^2\dd t\bigg]\\
    &=\frac{1}{2}\int_0^T \E\bigg[\beta_t \|\bA_t^{\star}\overrightarrow\bx_t - \nabla\log\overrightarrow\psi_t(\overrightarrow\bx_t)\|_2^2\bigg]\dd t.
\end{split}
\end{equation*}

Similarly, applying $(a+b)^2 \leq 2a^2 + 2b^2$, we have
\begin{equation*}
\begin{split}
    \text{KL}({\mathrm{L}}(\cdot)\|{\mathrm{L}}^{(k)}(\cdot)) &\leq \frac{3}{2}\int_0^T \E\bigg[\beta_t \big(\underbrace{\|\bA_t^{(k)}\overrightarrow\bx_t - \bA_t^{\star}\overrightarrow\bx_t\|_2^2}_{\text{convergence of SA}} +\underbrace{\|\bA_t^{\star}\overrightarrow\bx_t - \nabla\log\overrightarrow\psi_t(\overrightarrow\bx_t)\|_2^2}_{\text{variational gap based on $\bA_t^{\star}$}} \big)\bigg]\dd t \\
    &\lesssim \eta_k + \int_0^T \E\bigg[\beta_t \|\bA_t^{\star}\overrightarrow\bx_t - \nabla\log\overrightarrow\psi_t(\overrightarrow\bx_t)\|_2^2\bigg]\dd t.
\end{split}
\end{equation*}
\qed
\end{proof}




\section{Experimental Details}
\label{supp_exp_details}

\subsection{Parametrization of the Variational Score}

For the general transport, there is no closed-form update and we adopt an SVD decomposition with time embeddings to learn the linear dynamics in Figure \ref{linear_module}. The number of parameters is reduced by thousands of times, which have greatly reduced the training variance \citep{FFJORD}.

\begin{figure}[!ht]
  \centering
    \subfigure{\includegraphics[scale=0.35]{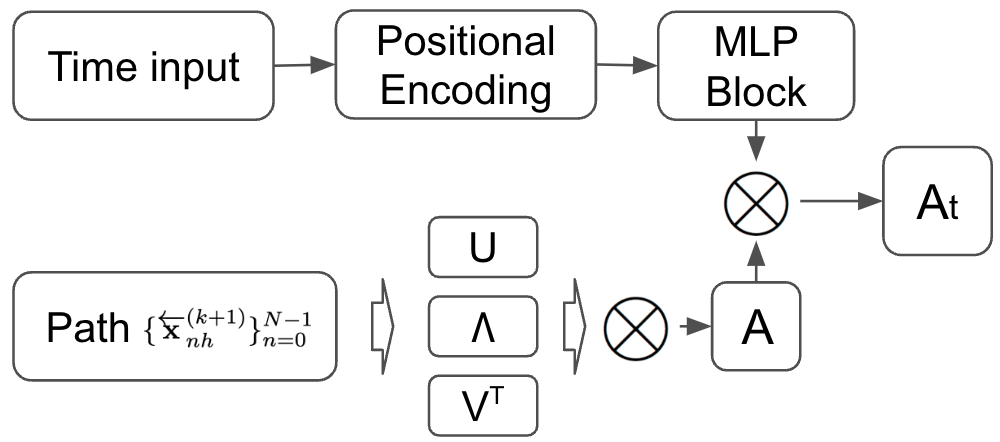}}
    \vspace{-0.1in}
  \caption{Architecture of the linear module. Both $U$ and $V$ are orthogonal matrices and $\Lambda$ denotes the singular values.}\label{linear_module}
  \vspace{-1em}
\end{figure}

\subsection{Synthetic Data}
\label{syn_appendix}

\subsubsection{Checkerboard Data}
\label{Checkerboard_data}
The generation of the checkerboard data is presented in Figure. \ref{VSDM_SGM_check}. The probability path is presented in Figure. \ref{trajectories_dynamics_check}. The conclusion is similar to the spiral data.

\begin{figure}[!ht]
  \centering
  \subfigure{\includegraphics[scale=0.13]{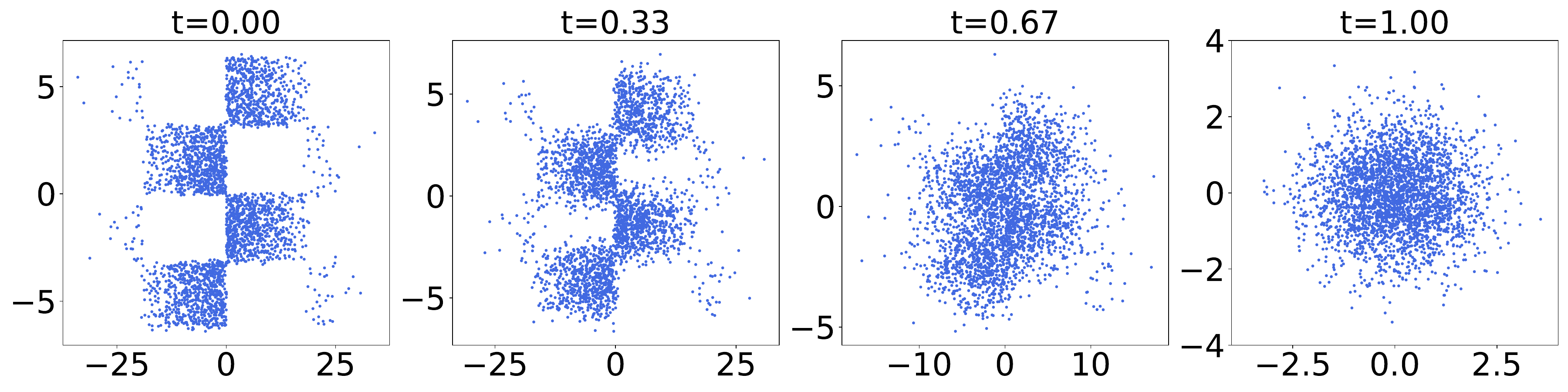}} 
    \subfigure{\includegraphics[scale=0.13]{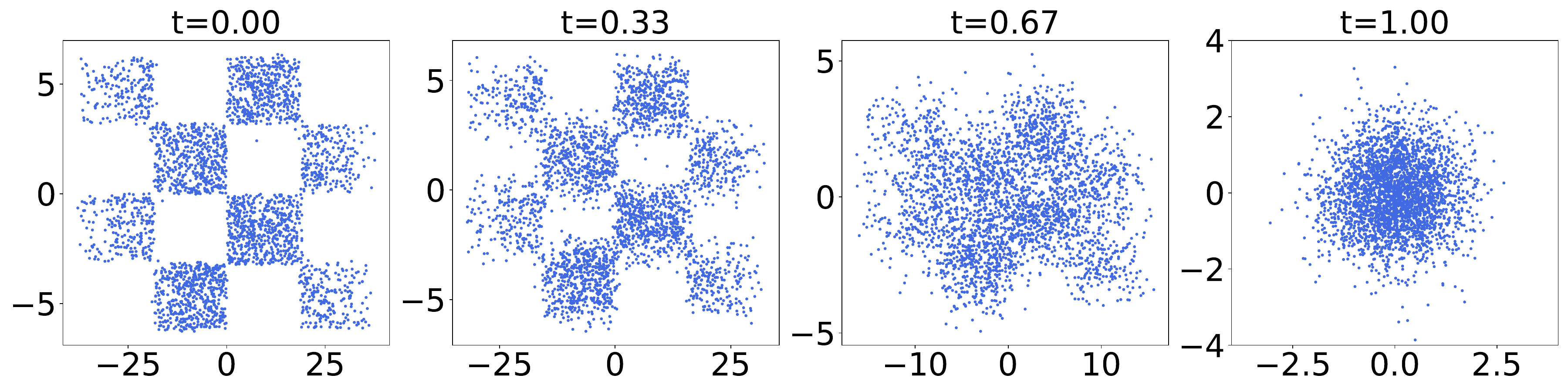}}
    \vspace{-0.05in}
  \caption{Variational Schr\"odinger diffusion models (VSDMs, right) v.s. SGMs (left) with the same hyperparameters ($\beta_{\max}=10$).}\label{VSDM_SGM_check}
  \vspace{-1em}
\end{figure}

\begin{figure*}[!ht]
  \centering
    \subfigure[\small{SGM-10}]{\includegraphics[scale=0.16]{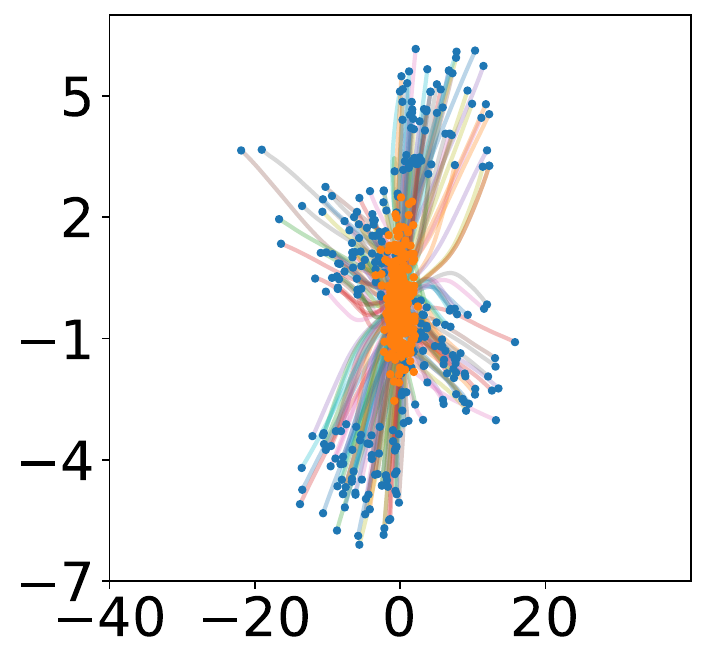}}
  \subfigure[\small{SGM-20}]{\includegraphics[scale=0.16]{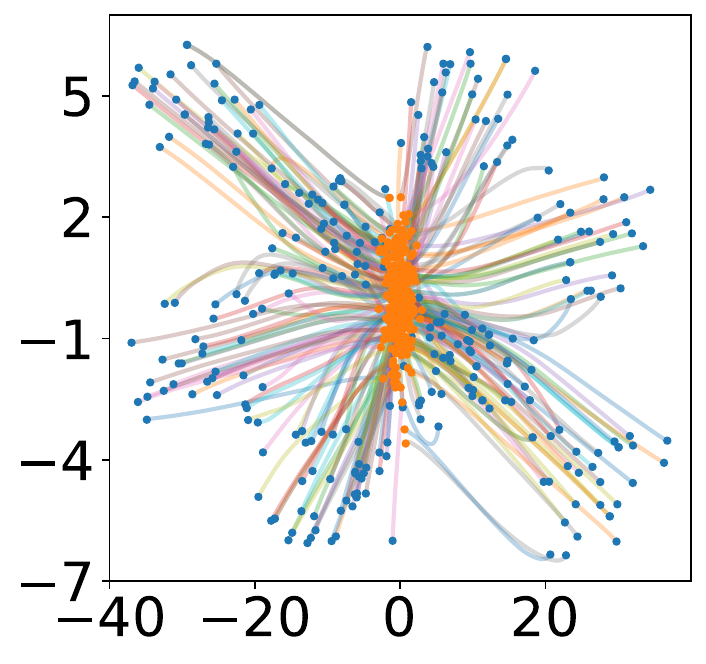}}
  \subfigure[\small{SGM-30}]{\includegraphics[scale=0.16]{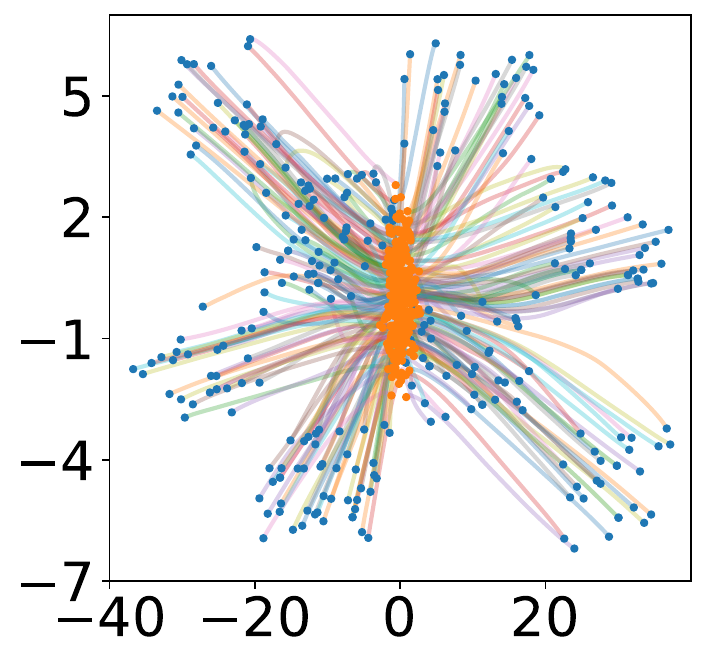}}
  \subfigure[\small{VSDM-10}]{\includegraphics[scale=0.16]{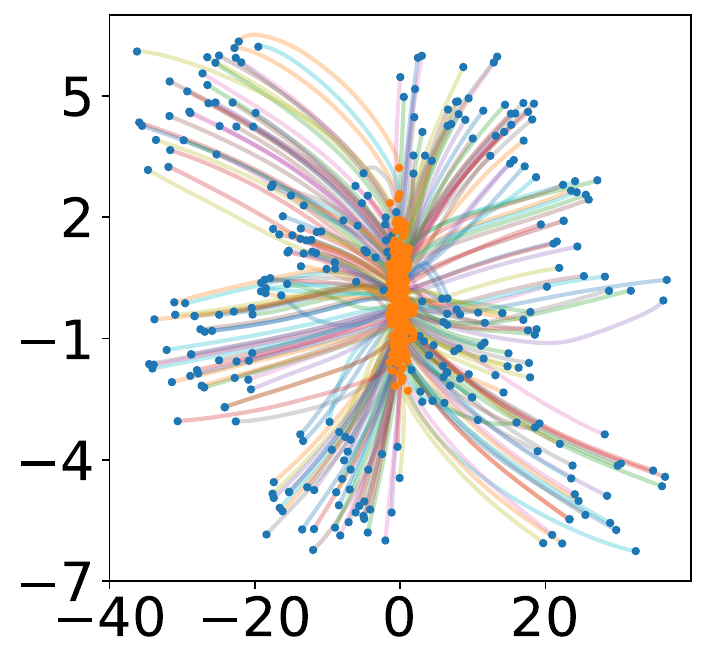}} \vskip -0.1in
  \caption{Probability flow ODE via VSDMs and SGMs. SGM with $\beta_{\max}=10$ is denoted by SGM-10 for convenience.}\label{trajectories_dynamics_check}
  \vspace{-1em}
\end{figure*}

\subsubsection{Convergence and Computational Time}

\paragraph{Convergence Study}
Under the same setup, VSDM-10 adaptively learns  $\bA_t$ (and $\bD_t$) on the fly and adapts through the pathological geometry via optimal transport. For the spiral-8Y data, the Y-axis of the singular values of $\bD_t$ (scaled by $\beta_{\max}$) converges from 10 to around 19 as shown in Figure \ref{eigenvalue_analysis_main}. The \textcolor{dark2blue}{singular value of the X-axis} quickly converges from 10 to a conservative scale of 7. We also tried VSDM-20 and found that both the Y-axis and X-axis converge to similar scales, which justifies the stability. 

\begin{figure}[!ht]
  \centering
    \subfigure[Spiral-8Y]{\includegraphics[scale=0.28]{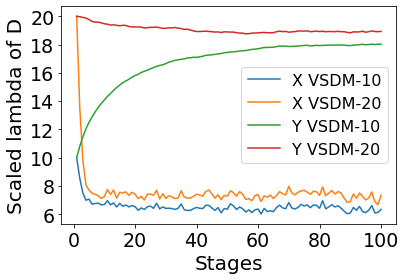}}
    \subfigure[Convergence v.s. time]{\includegraphics[scale=0.28]{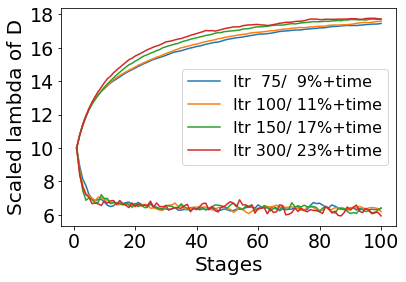}}
    \subfigure[Checkboard-6X]{\includegraphics[scale=0.28]{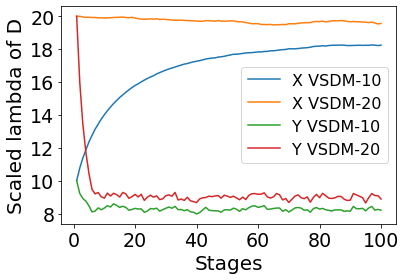}}
    \vspace{-0.1in}
  \caption{Optimization of $\bLambda$ of $\bD$ scaled by $\beta_{\max}$ (scaled lambda of $\bD$) of VSDM-10 and VSDM-20.}\label{eigenvalue_analysis_main}
  \vspace{-1em}
\end{figure}

\paragraph{Computational Time} We tried different budgets to train the variational scores and observed in Figure \ref{eigenvalue_analysis_main}(b) that 300 iterations yield the fastest convergence among the 4 choices but also lead to 23\% extra time compared to SGM. Reducing the number of iterations impacts convergence minimally due to the linearity of the variational scores and significantly reduces the training time.

\subsubsection{Evaluation of The Straightness}
\label{straightness_section}

Straighter trajectories lead to a smaller number of functional evaluations (NFEs). In section \ref{straightness_section}, we compare VSDM-20 with SGM-20 with NFE=6 and NFE=8 using the same computational budget and observe in Figure \ref{VSDM_SGM_NFE6} and \ref{VSDM_SGM_NFE8} the superiority of the VSDM model in generating more details.

To evaluate the straightness of the probability flow ODE, similar in spirit to \citet{Multisample_flow_matching}, we define our straightness metric by approximating the second derivative of the probability flow \eqref{backward_DSM_ode} as follows
\begin{equation}
\begin{split}
\label{straightness_metric}
    \text{S}(i)&=\int_0^T \E_{\overleftarrow{\bx}_t\sim \overleftarrow{\rho}_{t} }\bigg[\bigg|\frac{\dd^2 \overleftarrow{\bx}_t(i)}{\dd t^2}\bigg|\bigg] \dd t,\\
\end{split}
\end{equation}

where $i\in\{1,2\}$, $\overleftarrow{\bx}_t(1)$ and $\overleftarrow{\bx}_t(2)$ denote the X-axis and Y-axis, respectively. $S\geq 0$ and $S=0$ only when the transport is a straight path.

We report the straightness in Table \ref{straight_metric} and find the improvement of VSDM-10 over SGM-20 and SGM-30 is around 40\%. We also tried VSDM-20 on both datasets and found a significant improvement over the baseline SGM methods. However, despite the consistent convergence in Figure \ref{eigenvalue_analysis_main}, we found VSDM-20 still performs slightly worse than VSDM-10, which implies the potential to tune $\beta_{\max}$ to further enhance the performance.  

\begin{table}[ht]
\begin{sc}
\caption[Table caption text]{Straightness metric defined in Eq.\eqref{straightness_metric} via SGMs and VSDM with different $\beta_{\max}$'s. SGM with $\beta_{\max}=10$ (SGM-10) fails to generate data of anisotropic shapes and is not reported.}\label{straight_metric}
\small
\begin{center} 
\begin{tabular}{c|cc}
\hline
\small{Straightness (X / Y)} & \small{Spiral-8Y}  &  \scriptsize{Checkerboard-6X}  \\
\hline
\small{SGM-20} & 8.3 / 49.3 &  53.5 / 11.0  \\
\small{SGM-30} & 9.4 / 57.3 &  64.6 / 13.1 \\
\small{VSDM-20} & \textbf{6.3 }/ 45.6 & 49.4 / \textbf{$\ $ 7.4}  \\
\small{VSDM-10} & \textbf{5.5 / 38.7} & \textbf{43.9 / $\ $ 6.5}  \\
\hline
\end{tabular}
\end{center} 
\end{sc}
\end{table}


\subsubsection{A Smaller Number of Function Evaluations}

We also compare our VSDM-20 with SGM-20 based on a small number of function evaluations (NFE). We use probability flow to conduct the experiments and choose a uniform time grid for convenience. We find that both models cannot fully generate the desired data with NFE=6 in Figure \ref{VSDM_SGM_NFE6} and VSDM appears to recover more details, especially on the top and bottom of the spiral. For the checkboard data, both models work nicely under the same setting and we cannot see a visual difference. With NFE=8 in Figure \ref{VSDM_SGM_NFE8}, we observe that our VSDM-20 works remarkably well on both datasets and is slightly superior to SGM-20 in generating the corner details.

\begin{figure*}[!ht]
  \centering
    \subfigure{\includegraphics[scale=0.13]{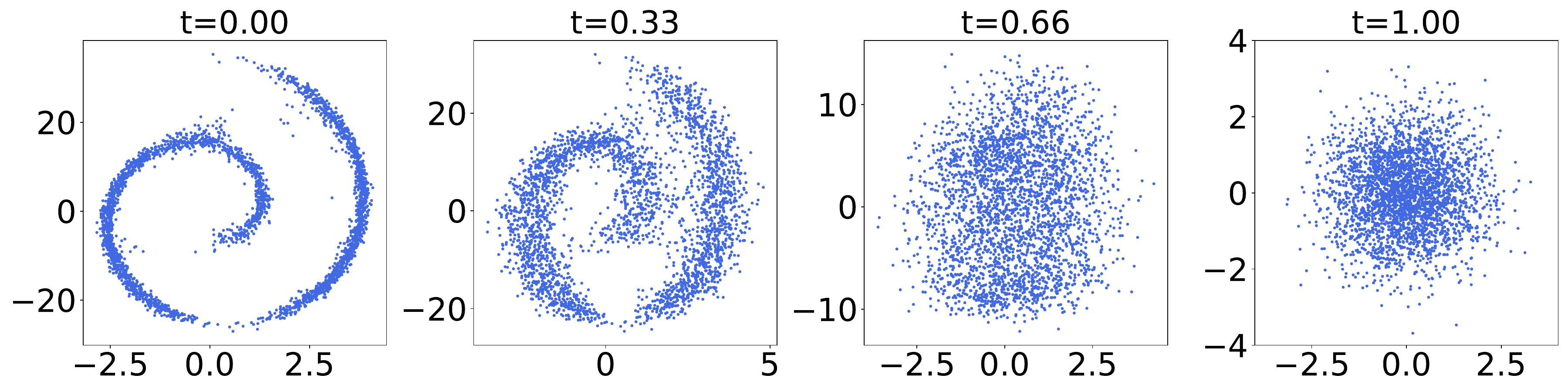}}
  \subfigure{\includegraphics[scale=0.13]{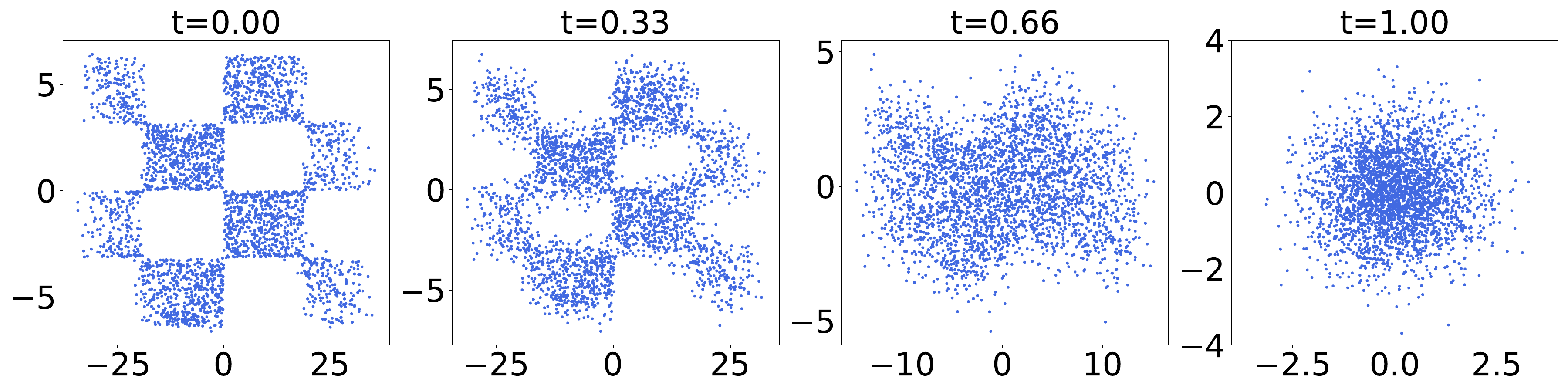}} \\
  \vspace{-0.18in}
    \subfigure{\includegraphics[scale=0.13]{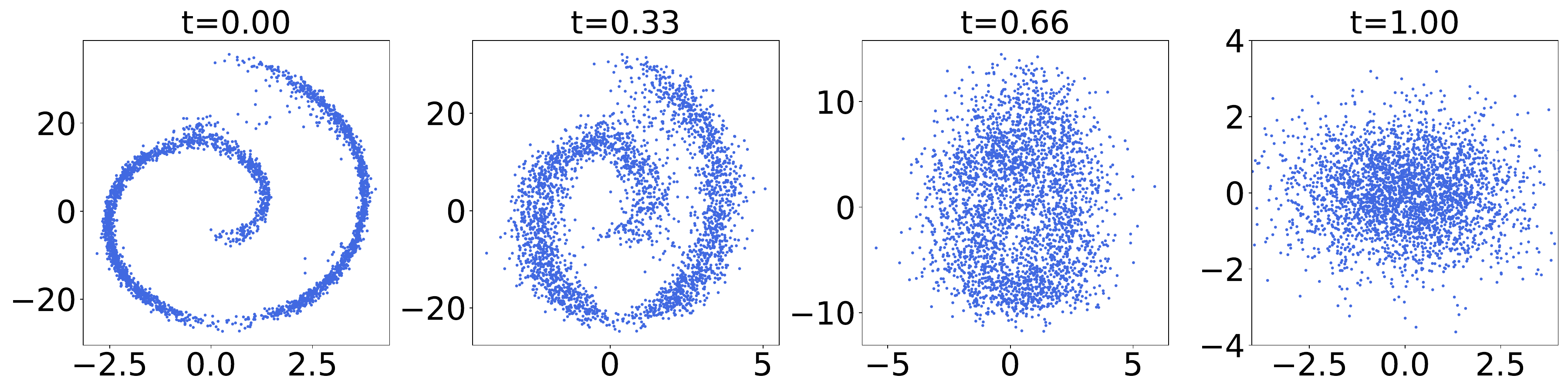}}
    \subfigure{\includegraphics[scale=0.13]{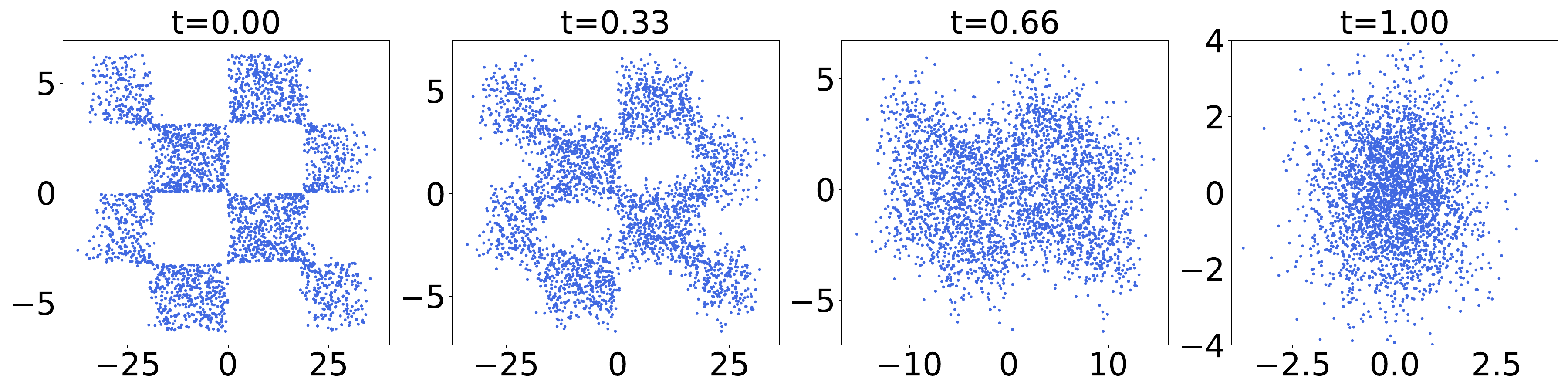}}
    \vspace{-0.15in}
  \caption{Variational Schr\"odinger diffusion models (bottom) v.s. SGMs (top) with the same hyperparameters ($\beta_{\max}=20$) and six function evaluations (NFE=6). Both models are generated by probability flow ODE.}\label{VSDM_SGM_NFE6}
  \vspace{-1em}
\end{figure*}

\begin{figure*}[!ht]
  \centering
    \subfigure{\includegraphics[scale=0.13]{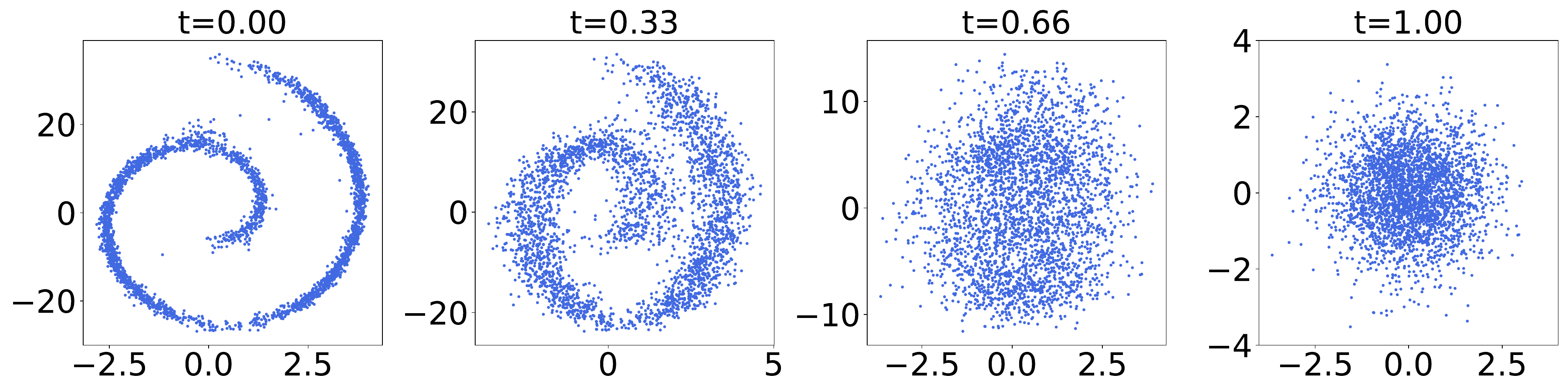}}
  \subfigure{\includegraphics[scale=0.13]{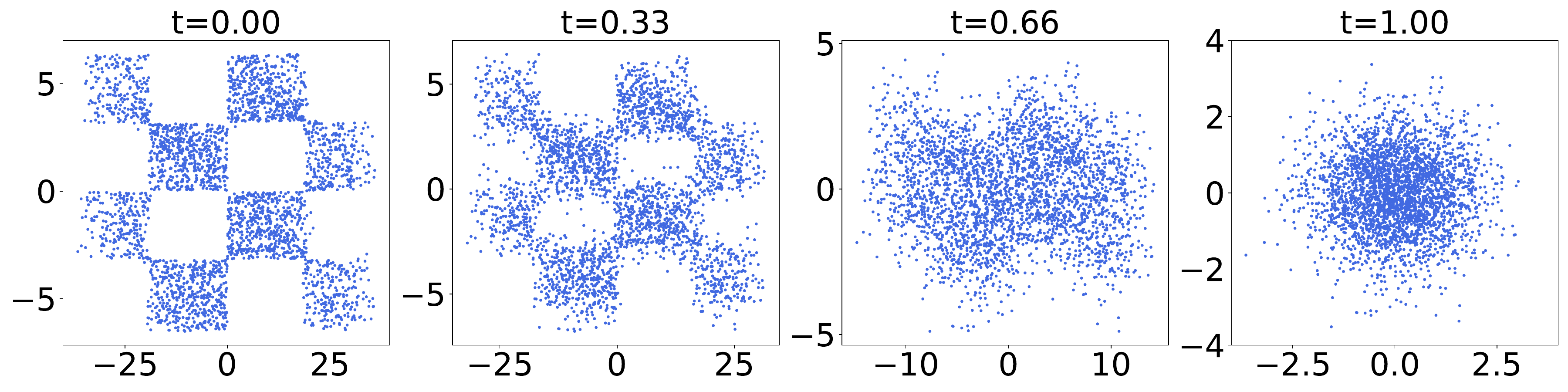}} \\
  \vspace{-0.18in}
    \subfigure{\includegraphics[scale=0.13]{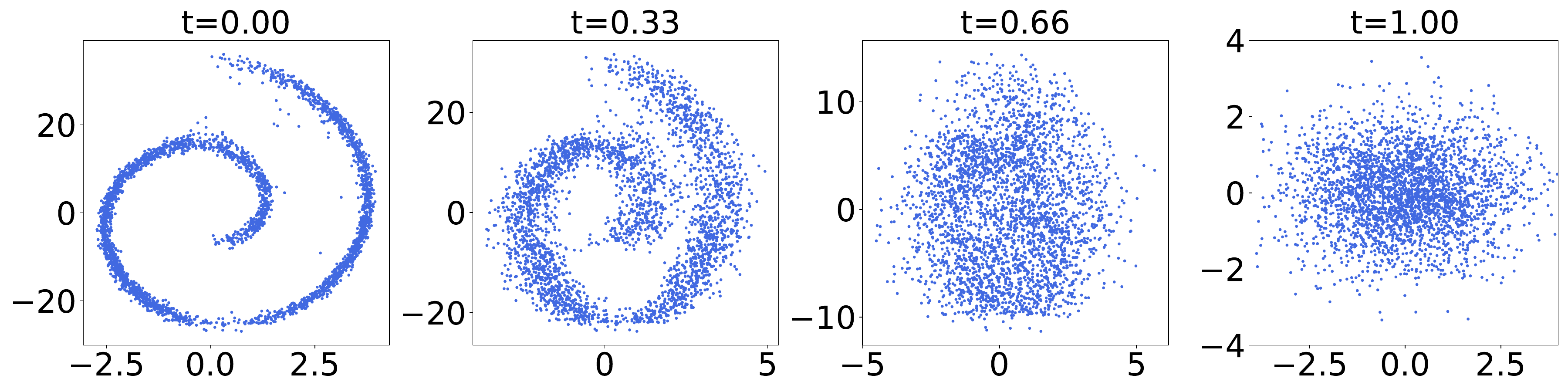}}
    \subfigure{\includegraphics[scale=0.13]{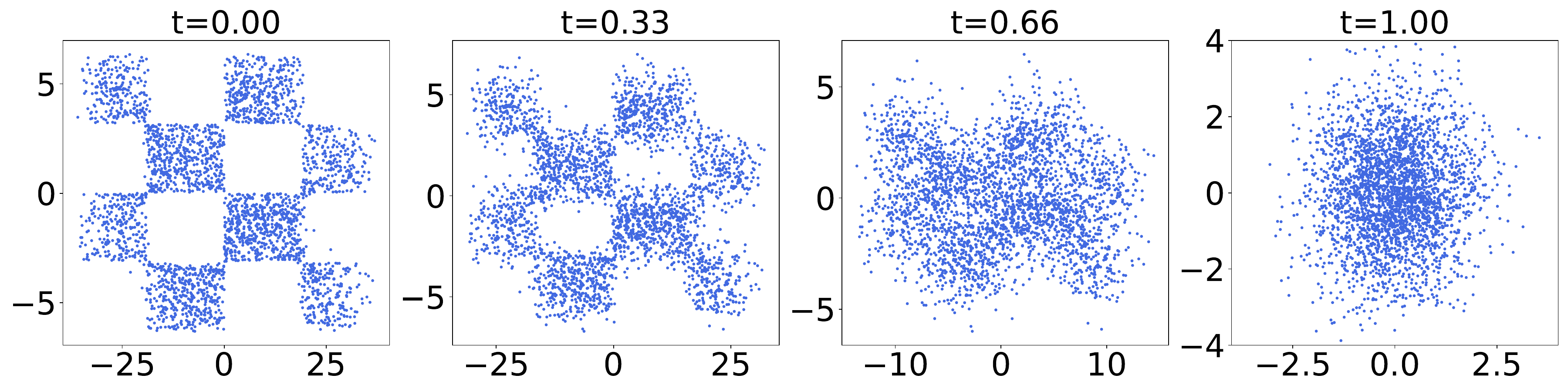}}
    \vspace{-0.15in}
  \caption{Variational Schr\"odinger diffusion models (bottom) v.s. SGMs (top) with the same hyperparameters ($\beta_{\max}=20$) and eight function evaluations (NFE=8). Both models are generated by probability flow ODE.}\label{VSDM_SGM_NFE8}
  \vspace{-1em}
\end{figure*}

\subsection{Multivariate Probabilistic Forecasting}\label{app:forecasting}

\paragraph{Data.} We use publicly available datasets. Exchange rate has 6071 8-dimensional measurements and a daily frequency. The goal is to predict the value over the next 30 days. Solar is an hourly 137-dimensional dataset with 7009 values. Electricity is also hourly, with 370 dimensions and 5833 measurements. For both, we predict the values over the next day.


\paragraph{Training.} We adopt the encoder-decoder architecture as described in the main text, and change the decoder to either our generative model or one of the competitors. The encoder is an LSTM with 2 layers and a hidden dimension size 64. We train the model for 200 epochs, where each epoch takes 50 model updates. In case of our model we also alternate between two training directions at a predefined rate. The neural network parameterizing the backward direction has the same hyperparameters as in \cite{rasul2021autoregressive}, that is, it has 8 layers, 8 channels, and a hidden dimension of 64. The DDPM baseline uses a standard setting for the linear beta-scheduler: $\beta_{\text{min}} = 0.0001$, $\beta_{\text{max}} = 0.1$ and 150 steps.


\begin{figure*}[!ht]
  \centering
    \subfigure{\includegraphics[scale=0.35]{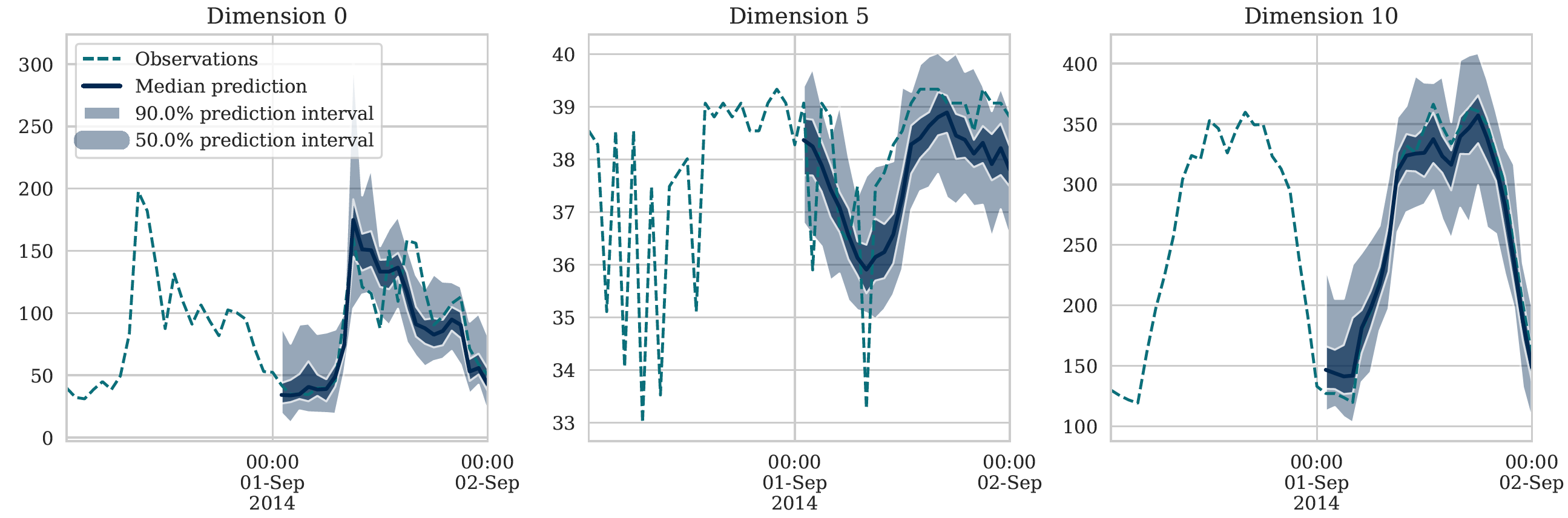}} \\
  \subfigure{\includegraphics[scale=0.53]{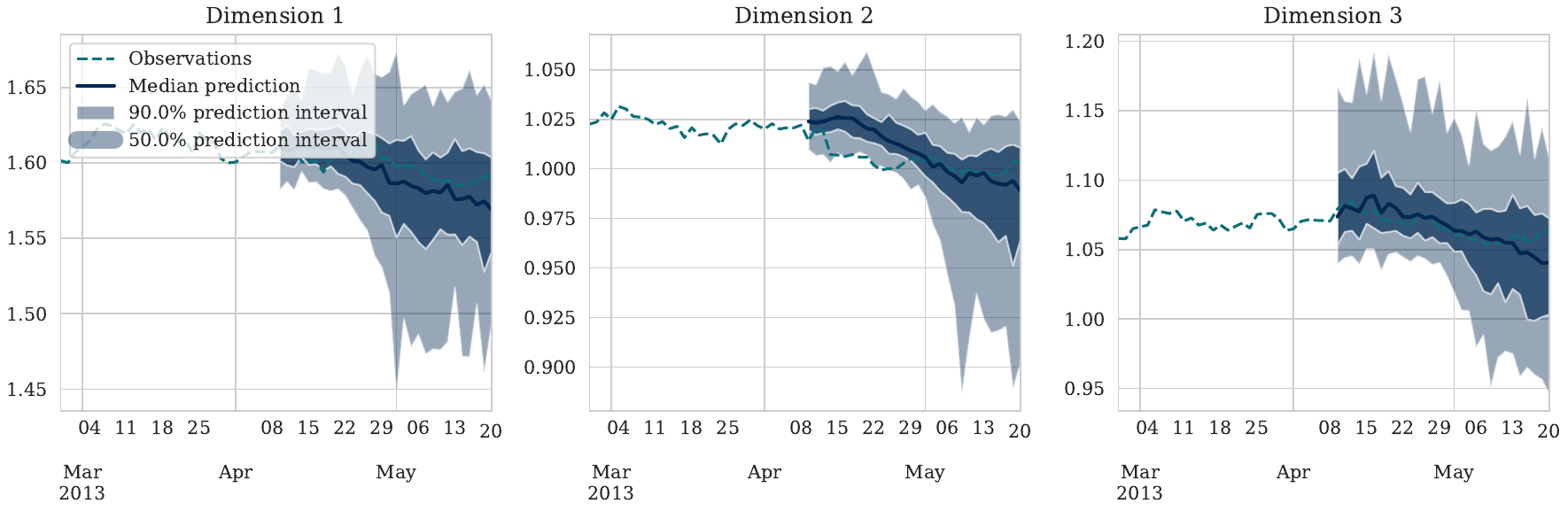}} \\
  \vspace{-0.18in}
    \quad \subfigure{\includegraphics[scale=0.54]{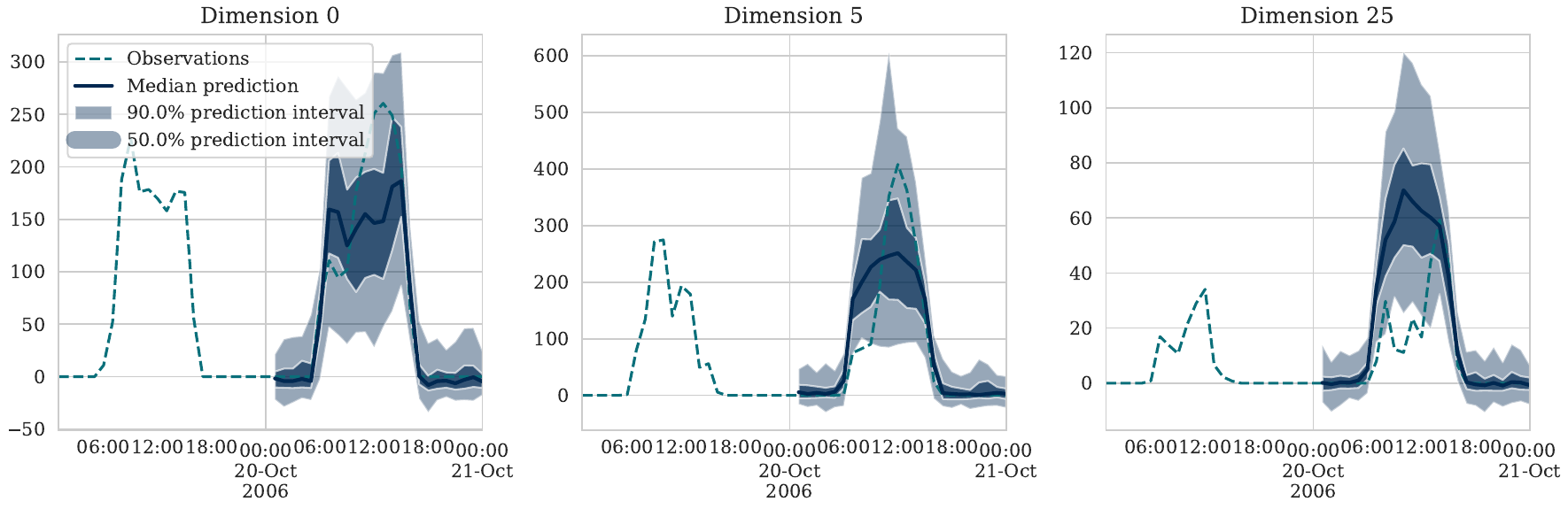}}
    \vspace{-0.15in}
  \caption{Example forecasts for Electricity (top), Exchange (middle), and Solar (bottom) datasets using our VSDM model. We show 3 out of 370, 8, and 137 dimensions, respectively.}\label{fig:electricity}
  \vspace{-1em}
\end{figure*}





\end{document}